\documentclass[aos, preprint]{imsart}
\usepackage{smile}
\usepackage{amsthm,amsmath,natbib}
\RequirePackage[colorlinks,citecolor=blue,urlcolor=blue]{hyperref}
\usepackage{graphicx}
\usepackage{microtype}
\usepackage{hyperref}
\usepackage[shortlabels]{enumitem}
\startlocaldefs
\endlocaldefs

\begin{document}

\begin{frontmatter}

\title{Sharp Rate of Convergence for Deep Neural Network Classifiers under the Teacher-Student Setting}

\runtitle{Sharp Classification in Teacher-Student Network}


\begin{aug}
\author{\fnms{Tianyang} \snm{Hu}\thanksref{m1}\ead[label=e1]{hu478@purdue.edu}},
\author{\fnms{Zuofeng} \snm{Shang}\thanksref{m3}
\ead[label=e3]{zuofeng.shang@njit.edu}}
\and 
\author{\fnms{Guang} \snm{Cheng}\thanksref{m1}\ead[label=e2]{chengg@purdue.edu}}


\affiliation{Purdue University\thanksmark{m1} and New Jersey Institute of Technology \thanksmark{m3}}


\end{aug}

\runauthor{T. Hu, Z. Shang and G. Cheng}

\begin{abstract}
Classifiers built with neural networks handle large-scale high dimensional data, such as facial images from computer vision, extremely well while traditional statistical methods often fail miserably. 
In this paper, we attempt to understand this empirical success in high dimensional classification by deriving the convergence rates of excess risk. 
In particular, a teacher-student framework is proposed that assumes the Bayes classifier to be expressed as ReLU neural networks. 
In this setup, we obtain a sharp rate of convergence, i.e., $\tilde{O}_d(n^{-2/3})$\footnote{The notation $\tilde{O}_d$ means ``up to a logarithmic factor depending on the data dimension $d$.''}, for classifiers trained using either 0-1 loss or hinge loss. 
This rate can be further improved to $\tilde{O}_d(n^{-1})$ when the data distribution is separable. Here, $n$ denotes the sample size. An interesting observation is that the data dimension only contributes to the $\log(n)$ term in the above rates. This may provide one theoretical explanation for the empirical successes of deep neural networks in high dimensional classification, particularly for structured data. 
\end{abstract}



\end{frontmatter}

\section{Introduction}
Deep learning has gained tremendous success in classification problems
such as image classifications \citep{deng2009imagenet}.
With the introduction of convolutional neural network \citep{krizhevsky2012imagenet} and residual neural network \citep{he2016deep}, various benchmarks in computer vision have been revolutionized and neural network based methods have achieved better-than-human performance \citep{nguyen2017iris}. 
For instance, AlexNet \citep{krizhevsky2012imagenet} and its variants \citep{zeiler2014visualizing,simonyan2014very} have demonstrated superior performance in ImageNet data \citep{imagenet_cvpr09, ILSVRC15},
where the data dimension is huge,
i.e., each image has pixel size $256\times 256$ and hence is an $65536$-dimensional vector.
Traditional statistical thinking sounds an alarm when facing such high-dimension data as the ``curse of dimensionality" usually prevents nonparametric classification achieving fast convergence rates. This work attempts to provide a theoretical explanation for the empirical success of deep neural networks (DNN) in (especially high dimensional) classification, beyond the existing statistical theories. 


In the context of nonparametric regression, similar investigations have been recently carried out. Among others \citep{farrell2018deep, suzuki2018adaptivity, nakada2019adaptive, oono2019approximation, chen2019efficient, liu2019optimal}, \cite{schmidt2017nonparametric} shows that deep ReLU neural networks can achieve minimax rate of convergence when the underlying regression function possesses a certain compositional smooth structure; \cite{bauer2019deep,kohler2019rate} show a similar result by instead considering hierarchical interaction models; \cite{imaizumi2018deep} demonstrate the superiority of neural networks in estimating a class of non-smooth functions, in which case no linear methods, e.g. kernel smoothing, Gaussian process, can achieve the optimal convergence rate. The aforementioned works all build on the traditional smoothness assumption and the convergence rates derived therein are still subject to curse of dimensionality.

Classification and regression are fundamentally different due to the discrete nature of class labels. Specifically, in nonparametric regression, we are interested in recovering the whole underlying function while in classification, the focus is on the nonparametric estimation of sets corresponding to different classes. As a result, it is well known that many established results on regression cannot be directly translated to classification. The goal of this paper is hence to fill this gap by investigating how well neural network based classifiers can perform in theory and further provide a theoretical explanation for the ``break-the-curse-of-dimensionality'' phenomenon. To this end, we propose to study neural network based classifiers in a teacher-student setting where the traditional smoothness assumption is no longer present. 

The teacher-student framework has originated from statistical mechanics \citep{Saad96dynamicsof, DBLP:journals/sac/MaceC98, engelb01} and recently gained increasing interest \citep{hinton2015distilling, NIPS2014_5484, NIPS2019_8921, NIPS2018_7584}. In this setup, one neural network, called student net, is trained on data generated by another neural network, called teacher net. While worst-case analysis for arbitrary data distributions may not be suitable for real structured dataset, adopting this framework can facilitate the understanding of how deep neural networks work as it provides an explicit target function with bounded complexity. Furthermore, assuming the target classifier to be a teacher network of an explicit architecture may provide insights on what specific architecture of the student classifier is needed to achieve an optimal excess risk. At the same time, by comparing the two networks, both optimization and generalization can be handled more elegantly. Existing works on how well student network can learn from the teacher mostly focus on regression problems and study how the student network evolves during training from computational aspects, e.g., \citep{tian2018theoretical, tian2019over, NIPS2019_8921, zhang2018learning, cao2019tight}. Still, there is a lack of statistical understanding in this important direction, particularly on classification aspects. 

In this paper, we consider binary classification, and focus on the teacher-student framework where the optimal decision region is defined by ReLU neural networks. 
This setting is closely related to the classical smooth boundary assumption where the neural networks are substituted by smooth functions. 
Specifically, a well-adopted assumption called as ``boundary fragment'' \citep{mammen1999smooth, tsybakov2004optimal,imaizumi2018deep} assumes the smooth function to be linear in one of the dimensions (see Appendix \ref{app:boundary}).
Our teacher-student network setting is more general as it does not impose any special structures on the decision boundary. Moreover, by the universal approximation property \citep{cybenko1989approximations, arora2016understanding, lu2017expressive}, the teacher network can sufficiently approximate any continuous function given large enough size.

In the above setting, an un-improvable rate of convergence  is derived as $\tilde{O}_d(n^{-2/3})$ for the excess risk of the empirical 0-1 loss minimizer, given that the student network is deeper and larger than the teacher network (unless the teacher network has a limited capacity in some sense to be specified later). When data are separable, the rate improves to $\tilde{O}_d(n^{-1})$. 
In contrast, under the smooth boundary assumption, \cite{mammen1999smooth} establish the optimal risk bound to be $O(n^{-\beta(\kappa+1)/[\beta(\kappa+2)+(d-1)\kappa]})$ where $\beta>0$ represents the smoothness of the ``boundary fragments'' and $\kappa>0$ is the so-called Tsybakov noise exponent. Clearly, this rate suffers from the ``curse of dimensionality" but interestingly, coincides with our rate when $\kappa=1$ and $\beta\to\infty$ (up to a logarithmic factor). If we further allow $\kappa\to\infty$ (corresponding to separable data), the classical rate above recovers $\tilde{O}(n^{-1})$ (up to a logarithmic factor). Please see the Appendix \ref{app:boundary} for detail. 

Furthermore, we extend our analysis to a specific surrogate loss, i.e., hinge loss, and show that the convergence rate remains the same (up to higher order logarithmic terms) while allowing {\em deeper} student and teacher nets. The obtained sharp risk bounds may explain the empirical success of deep neural networks in high-dimensional classification as the data dimension $d$ only appears in the $\log(n)$ terms. Our main technical novelty is the nontrivial entropy calculation for nonparametric set estimation based on combinatorial analysis of ReLU neural networks. 

\paragraph{Existing Works} We review some related works on classification using neural networks. The first class of literature demonstrate that deep neural network (DNN) classifiers can be efficiently optimized in different senses. For example, \cite{liang2018understanding} study the loss surface of neural networks in classification and provide conditions that guarantee zero training error at all local minima of appropriately chosen surrogate loss functions. Additionally, \cite{lyu2019gradient} show that under exponential loss \citep{soudry2018implicit,gunasekar2018implicit}, gradient descent on homogeneous neural network has implicit bias towards the maximum $L_2$ margin solution. In all these optimization works, sharp bounds are not derived for either generalization error or convergence rate of the excess risk. From a nonparametric perspective, \cite{kim2018fast} derive the excess risk bound for DNN classifiers as $\tilde{O}(n^{-\beta(\kappa+1)/[\beta(\kappa+2)+(d-1)(\kappa+1)]}),$ which is suboptimal in the sense of \cite{tsybakov2004optimal}. Clearly, such a dimension-dependent bound, indicating exponential dependence on the sample size, may not support the empirical success of deep learning in high dimension classification.



\paragraph{Notations}
For any function $f:\cX\to\RR$, denote 
let $\|f\|_\infty=\sup_{\bx\in\cX}|f(\bx)|$ and $\|f\|_p=(\int |f|^p)^{1/p}$ for $p\in \NN$.
For two given sequences $\{a_n\}_{n\in \NN}$ and $\{b_n\}_{n\in \NN}$ of real numbers, we write $a_n\lesssim b_n$ if there exists a constant $C>0$ such that $a_n\le C b_n$ 
for all sufficiently large $n$, which is also denoted as $a_n=O(b_n)$. 
Let $\Omega(\cdot)$ be the counterpart of $O(\cdot)$ that $a_n=\Omega(b_n)$ means $a_n\gtrsim b_n$ 
In addition, we write $a_b\asymp b_n$ if $a_n\lesssim b_n$ and $a_n\gtrsim b_n$. 
For $a,b \in \RR$, denote $a\vee b =\max\{a, b\}$ and $a\wedge b =\min\{a, b\}$. Let $\lfloor a \rfloor$ represent $\max\{n\in\NN: n\le a\}$.
$\II$ denotes the indicator function. Independently, identically distributed is abbreviated as i.i.d..

\section{Preliminary}
\subsection{Neural Network Setup}
\label{sec:DNN}
We consider deep neural networks with Rectified Linear Unit (ReLU) activation that $\sigma(x) = \max\{x,0\}$. For an $L$ hidden layer ReLU neural network $f(\cdot)$, let the width of each layer be $n_0,n_1,\cdots,n_L$, where $n_0=d$ is the input dimension, and denote the weight matrices and bias vectors in each layer to be $\bW^{(l)}$ and $\bb^{(l)}$, respectively. Let $\sigma_{(\bW,\bb)}(\bx)=\sigma(\bW\cdot \bx+ \bb)$ and $\circ$ represent function composition. Then, the ReLU DNN can be written as
$$f(\bx|\Theta)=\bW^{(L+1)}\sigma_{(\bW^{(L)},\bb^{(L)})}\circ\cdots\circ \sigma_{(\bW^{(1)},\bb^{(1)})}(\bx) + \bb^{(L+1)},\,\,\, \bx \in \mathbb{R}^{d},$$
where $\Theta=\{(\bW^{(l)}, \bb^{(l)})\}_{l=1,\dots, L+1}$ denotes the parameter set. 

For any given $\Theta$, let $|\Theta|$ be the number of hidden layers in $\Theta$, and $N_{\max}(\Theta)$ be the maximum width. We define $\|\Theta\|_0$ as the number of nonzero parameters:
$$\|\Theta\|_0=\sum_{l=1}^{L+1}\left( \|\text{vec}(\bW^{(l)})\|_0 +\|\bb^{(l)}\|_0\right),$$
where $\text{vec}(\bW^{(l)})$ transforms the matrix $\bW^{(l)}$ into the corresponding vector by
concatenating the column vectors.
Similarly, we define $\|\Theta\|_\infty$ as the largest absolute value of the parameters in $\Theta$,
$$\|\Theta\|_\infty = \max \left\{ \max_{1\le l\le L+1} \|\text{vec}(\bW^{(l)})\|_\infty, 
\max_{1\le l\le L+1} \|\bb^{(l)}\|_\infty\right\}.$$ 


For any given $n$, let $\cF_n$ be
    \begin{align*}
    \cF_n &= \cF^{\textup{DNN}}(L_n, N_n, S_n, B_n, F_n)\\
    &=\big\{f(\bx|\Theta):  |\Theta|\le L_n, N_{\max}(\Theta)\le N_n, \|\Theta\|_0\le S_n, \\
    	&\qquad\qquad\qquad \|\Theta\|_\infty\le B_n, \|f(\cdot|\Theta)\|_\infty \le F_n\big\}.
    \end{align*}

\subsection{Binary Classification}
Consider binary classification with a feature vector $\bx\in\cX\subset\RR^d$ and a label $y\in\{-1, 1\}$. 
Assume $\bx|y=1\sim p(\bx), \bx|y=-1\sim q(\bx)$ where $p$ and $q$ are two bounded densities on $\cX$ w.r.t. base measure $\QQ$. If $p,q$ have disjoint support, we say the data distribution or the classification problem is separable.
For simplicity, assume that $\QQ$ is Lebesgue measure, and positive and negative labels are equally likely to appear, i.e., balanced labels. 


The objective of classification is to find an optimal classifier (called the Bayes classifier) $C^*$ within some classifier family $\cC$, that minimizes the 0-1 loss defined as 
$$C^*=\argmin_{C\in \cC}R(C):=\argmin_{C\in \cC} \EE \left[ \II\{C(\bx)\ne y\}\right].$$
We can estimate $C^*$ based on the training data by minimizing the empirical 0-1 risk as follows
\begin{align*}
\hat{C}_n=\argmin_{C\in \cC_n}R_n(C):=\argmin_{C\in \cC_n} \sum_{i=1}^n \II\{C(\bx_i)\ne y_i\}/n,
\end{align*}
where $\cC_n$ is a given class of classifiers possibly depending on the sample size $n$.
In practice, the above 0-1 loss is often replaced by its (computationally feasible) surrogate counterparts \citep{bartlett2006convexity}, such as hinge loss $\phi(z) = (1-z)_+ = \max\{1-z, 0\}$ or logistic loss $\phi(z) = \log(1+\exp(-z))$.

Given a surrogate loss $\phi$, we first obtain $\hat{f}_{\phi}$ by minimizing the empirical risk
\begin{align*}
R_{\phi,n}(f)=\sum_{i=1}^n \phi( y_i f(\bx_i))/n
\end{align*}
over $\cF$, and then construct a classifier by $\hat{C}_{\phi}(\bx)={\rm sign} (\hat{f}_{\phi}(\bx)).$ 
Accordingly, define an optimal $f^*_{\phi}$ as
$
f^*_{\phi}= \argmin_{f\in\cF} R_{\phi}(f),
$
where $R_\phi(f):=\EE R_{\phi,n}(f)$ is the population risk. Given that $C(\bx) = \mbox{sign}(f(\bx)) $, with a slight abuse of notation, we write $R(C)$ and $R(f)$ interchangeably. 
A classifier $C$ is evaluated by its excess risk defined as the difference of the population risk between $C$ and the Bayes optimal classifier $C^*$ that
$\cE(C, C^*) = R(C) - R(C^*).$
Our goal is to derive sharp convergence rates of $\cE(C, C^*)$ under different losses.

\section{Teacher-Student Framework for Classification}

In this section, we set up the teacher-student framework for classification under which sharp rates for the excess risk are developed. Such a teacher-student bound sets an algorithmic independent benchmark for various deep neural network classifiers and also helps understand the role of input data dimension in the classification performance.
The Bayes classifier $C^*$ is defined via the optimal decision region $G^* := \{\bx\in\cX, p(\bx)- q(\bx)\ge 0\}$. The set estimate $\hat{G} = \{\bx\in\cX, \hat{f}(\bx)\ge 0\}$ can be constructed through deep neural network classifiers $\hat{f}:\RR^d\to\RR$ trained using either 0-1 loss or surrogate loss. Accordingly, a natural teacher network assumption is that $p(\bx)- q(\bx)$ can be expressed by some neural network $f_n^*\in \cF^*_n$.
Here, the underlying densities are indexed by $n$, but such an assumption is not uncommon in high-dimensional statistics, where population quantities may depend on the sample size $n$, e.g., \cite{zhao2006model}.
 
\subsection{Training with 0-1 Loss}

In this section, we focus on, for the theoretical purpose, DNN classifiers trained with the empirical 0-1 loss. Denote
$$\hat{f}_{n} = \argmin_{f\in \cF_n}\frac{1}{n}\sum_{i = 1}^n \II\{y_if(\bx_i)<0\},$$
given a certain DNN family $\cF_n$. 

It is important to control the complexity of the underlying classification problem. Otherwise, the student network would not be able to recover the Bayes classifier \citep{telgarsky2015representation} with sufficient accuracy. 
To this end, we impose the following teacher network assumptions on $(p(\bx)-q(\bx))$:
\begin{enumerate}
    \item[(A1)] 
    $p,q$ have compact supports.
    \item[(A2)]
     $p(\bx) - q(\bx)$ is representable by some teacher ReLU DNN $f^*_n\in\cF^*_n$ with 
    $$N_n^*=O\left(\log n\right)^{m_*},\quad L_n^*=O\left(1\right) \quad\mbox{for some } m_*\ge 1.$$
    
    \item[(A3)]
    For any $n$, there exists $c_n, 1/T_n=O(\log n)^{m^*d^2L_n^*}$ such that for all $0\le t\le T_n$.
    \begin{equation*}
        \QQ\{\bx\in\cX: |f_n^*(\bx)|\le t\}\le c_n t
    \end{equation*}
\end{enumerate}  

Assumption (A3) characterizes how concentrated the data are around the decision boundary, which can be seen as an extension to the classical Tsybakov noise condition \citep{mammen1999smooth}. The difference is that in our case, the underlying densities are indexed by sample size and thus $c_n$ and $T_n$ are allowed to vary with $n$. Assumption (A3) is not unrealistic as we will show that it holds with high probability if the teacher network is random as stated in the following lemma (see Appendix \ref{app:a3} for detail).

\begin{lemma}
\label{lemma:a3}
Let $f_n^*$ be the teacher network with structures specified in assumption (A2). Suppose that all weights of $f_n^*$ are i.i.d. with any continuous distribution, e.g. Gaussian, truncated Gaussian, etc.. Then, with probability at least $1-\delta$, assumption (A3) holds with $c_n, 1/T_n\le A(\delta)(\log n)^{m^*d^2L_n^*}$ where $A(\delta)$ is some constant depending on $\delta$.
\end{lemma}

The following theorem characterizes how well the student network of proper size can learn from the teacher in terms of the excess risk. 
\begin{theorem}
\label{thm:ts}
Under the teacher assumptions (A1) through (A3), 
denote all such $(p,q)$ pairs to be $\cP^*_n$ and let the corresponding Bayes classifier be $C_n^*$.
Let $\cF_n$ be a student ReLU DNN family with 
$N_n=O(\log n)^m$ and $L_n=O(1)$ for some $m\ge m_*$ and assume the student network is larger than the teacher network in the sense that $L_n\ge L_n^*, S_n\ge S_n^*, N_n\ge N_n^*, B_n\ge B_n^*$. Then the excess risk for $\hat{f}_n\in\cF_n$ satisfies 
$$\sup_{(p, q)\in{\cP}^*_n}\EE[\cE(\hat{f}_n,C^*_n)]=\tilde{O}_d 
	\left(\frac{1}{n}\right)^{\frac{2}{3}},$$
	where $\tilde{O}_d$ hides the $\log n$ terms, which depend on $d$.
\end{theorem}
The dependence on the dimension $d$ is in the order of $O(\log n)^{d^2}.$
We further argue that under the present setting, the rate $n^{-2/3}$ in Theorem \ref{thm:ts} cannot be further improved.
\begin{theorem}
\label{thm:lb}
Under the same assumptions of $p, q$ as in Theorem \ref{thm:ts} that $(p,q)\in\tilde{\cF}^*_n$. 
Let $\tilde{\cF}_n$ be an arbitrary function space, then 
$$\inf_{\tilde{f}_n\in \tilde{\cF}_n}\sup_{(p, q)\in\tilde{\cF}^*_n}\EE[\cE({f}_n,f^*_n)]=\tilde{\Omega}_d 
	\left(\frac{1}{n}\right)^{\frac{2}{3}},  $$
	where $\tilde{\Omega}_d$ hides the $\log n$ terms, which depend on $d$.
\end{theorem}
Theorem \ref{thm:lb} shows that the convergence rate achieved by the empirical 0-1 loss minimizer cannot be further improved (up to a logarithmic term).
If $p$ and $q$ have disjoint supports, i.e. separable, which could be true in some image data, the rate improves to $n^{-1}$, as stated in the following corollary. This rate improvement is not surprising since the classification task becomes much easier for separable data. 
\begin{corollary}
\label{cor:dis}
Under the same setting as in Theorem \ref{thm:ts}, if we further assume $p, q$ have disjoint supports, then the rate of convergence of the empirical 0-1 loss minimizer improves to 
$$\inf_{{f}_n\in \cF_n}\sup_{(p, q)\in\tilde{\cF}^*_n}\EE[\cE({f}_n,f^*_n)]\asymp
	\tilde{O}\left(\frac{1}{n}\right). $$
\end{corollary}
\begin{remark}[Disjoint Support]
\label{remark:thm1} 
Given that data are separable, \cite{srebro2010optimistic} derived the excess risk bound as $O(D\log n/n)$ (under a smooth loss) where $D$ is the VC-subgraph-dimension of the estimation family. Additionally, separability implies that the noise exponent $\kappa$ in Tsybakov noise condition \citep{mammen1999smooth,tsybakov2004optimal} can be arbitrarily large, which also gives $O(1/n)$ rate under the ``boundary fragments" assumption.
\end{remark} 

The imposed relation between teacher and student nets in Theorem \ref{thm:ts} is referred to as ``over-realization'' in \cite{NIPS2019_8921, tian2019over, bai19}: at each layer, the number of student nodes is larger than that of teacher nodes given the same depth. In other words, the student network is larger than the teacher in order to obtain zero approximation error. 

On the other hand, such a requirement is not necessary as long as the corresponding Bayes classifier is not too complicated. A ReLU neural network is a continuous piecewise linear function, i.e. its domain can be divided into connected regions (pieces) within where the function is linear. If the ReLU neural network crosses 0 on one piece, we call that piece as being {\em active} (see Figure \ref{fig:active} for an illustration). One way to measure the complexity of the teacher network is the number of active pieces. The following Corollary says that the teacher network can be much larger and deeper as long as the number of active pieces are in a logarithmic order with respect to $n$. 
\begin{figure}
    \centering
    \includegraphics[scale=0.5]{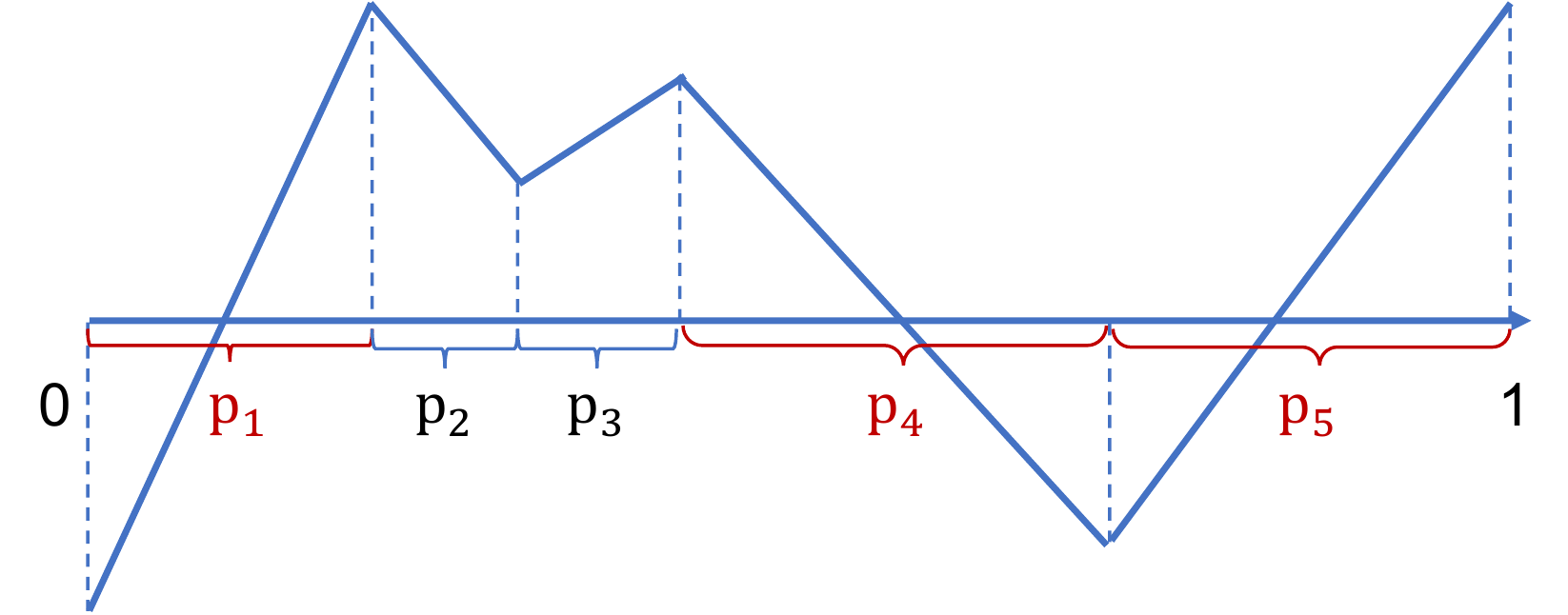}
    \caption{Example of a ReLU DNN function in $[0,1]$. There are 5 pieces $p_1,p_2,\ldots,p_5$ and among them, only $p_1,p_4,p_5$ cross value 0 (horizontal line). There are 3 active pieces in this example and they are colored red.}
    \label{fig:active}
\end{figure}
\begin{corollary}
\label{cor:active}
The same result in Theorem \ref{thm:ts} holds when the teacher network is larger than the student network, i.e. $L_n\le L_n^*, S_n\le S_n^*, N_n\le N^*_n, B_n\le B_n^*$ , given that the total number of active pieces in the teacher network is of the following order 
 \begin{equation}
 \label{eqn:active}
     o\left(\left( \prod_{l=1}^{L_n-1} \left\lfloor \frac{n_l}{d}
            \right\rfloor^{d} \right)
            \sum_{j=0}^{d} { n_{L_n} \choose j }\right),
 \end{equation}
where $n_1,\cdots,n_{L_n}$ are the width of each hidden layer of the student network.
\end{corollary}

The number of active pieces is the key quantity in controlling the complexity of the optimal set $G^*$. The expression in (\ref{eqn:active}) comes from the lower bound developed by \cite{montufar2014number} on the maximum number of linear pieces for a ReLU neural network (Lemma \ref{lemma:lb}). This lower bound is determined by the structure of the student network. If the number of active pieces of the teacher network is on this order, i.e. within the capacity of the student, then the corresponding optimal set can still be recovered by an even smaller student network, which ensures zero approximation error. Since the student network in consideration satisfies $N_n=O(\log n)^m$, the required order for the number of active pieces is in the order of $o(\log n)^{mdL_n}$.


\subsection{Proof of Theorem \ref{thm:ts}}
We first present some preliminary lemmas. 
As we mentioned before, classification can be thought of as nonparametric estimation of sets. 
For this, we define two distances over sets.
The first one is the usual symmetric difference of sets: for any $G_1,G_2\subset\RR^d$,
$$
d_\triangle(G_1, G_2)= \QQ(G_1\triangle G_2)= \QQ\left((G_1\backslash G_2)\cup (G_2\backslash G_1)\right),
$$
where $\QQ$ denotes the Lebesgue measure. The second one is induced by densities $p,q$:
for any $G_1,G_2\subset\RR^d$,
$$
d_{p, q}(G_1,G_2) =\int_{G_1\triangle G_2}|p(\bx) - q(\bx)|\QQ(d\bx).
$$

There are two key factors governing the rate of convergence in classification:
\begin{itemize}
    \item How concentrated the data are around the decision boundary;
    \item
    The complexity of the set $\cG^*$ where the optimal $G^*$ resides. 
\end{itemize}
For the first factor, the following Tsybakov noise condition \citep{mammen1999smooth} quantifies how close $p$ and $q$ are:
\begin{itemize}
    \item[(N)] There exists constant $c>0$ and $\kappa\in[0,\infty]$ such that for any  $0\le t\le T$
    $$\PP\left(\{\bx:|p(\bx)-q(\bx)|\le t\}\right)\le c t^\kappa.$$
\end{itemize}
The parameter $\kappa>0$ is referred to as the \textit{noise exponent}. The bigger the $\kappa$, the less concentrated the data are around the decision boundary and hence the easier the classification. 
In the extreme case that $p, q$ have different supports, $\kappa$ can be arbitrarily large and the classification is easy. To another extreme where $\QQ\{\bx\in\cX: p(\bx)=q(\bx)\}>0$, there exists a region where different classes are indistinguishable. In this case, $\kappa = 0$ and the classification is hard in that region.


For the second factor, we use bracketing entropy to measure the complexity of a collection of subsets $\cG$ in $\RR^d$. 
For any $\delta>0$, the bracketing number $\cN_B(\delta, \cG, d_\triangle)$ is the minimal number of set pairs $(U_j, V_j)$ 
such that, (a) for each $j$, $U_j\subset V_j$ and $d_\triangle(U_j, V_j)\le \delta$;
(b) for any $G\in \cG$, there exists a pair $(U_j, V_j)$ such that $U_j\subset G\subset V_j$.
Simply denote $\cN_B(\delta)=\cN_B(\delta, \cG, d_\triangle)$ if no confusion arises.
The bracketing entropy is defined as $H_B(\delta)=\log{\cN_B(\delta, \cG, d_\triangle)}$. 

Lemma~\ref{lemma:bracketing} characterizes the complexity of a special collection of sets.


\begin{lemma}
\label{lemma:bracketing}
Let $\cX=[0,1]^d$ and $\cG$ be a collection of polyhedrons with at most $S$ vertices in $\RR^d$. Then the bracketing entropy of $\bar{\cG}=\cG\cap \cX$ satisfies
$$H_B(\delta, \bar{\cG}, d_\triangle)=\log \cN_B(\delta,\bar{\cG}, d_\triangle)\lesssim d^2 S\log(d^{3/2}S/\delta)$$
\end{lemma}
\begin{proof}[Proof of Lemma \ref{lemma:bracketing}]
Let's first introduce some notations and terminologies. 
For any $\delta>0$, let $M_\delta$ denote the smallest integer such that $M_\delta>1/\delta$.
Consider the set of lattice points $\bX_\delta^d=\{(i_1/M_\delta,\ldots,i_d/M_\delta):
i_1,\ldots,i_d=0,1,\ldots,M_\delta\}$ which has cardinality $(M_{\delta}+1)^d$. Let $G(\bx_1, \cdots, \bx_s)$ denote a polyhedron with vertices $\bx_1, \cdots, \bx_s\in[0,1]^d$ where $s\le S$. (the $\bx_i$'s are not necessarily distinct). Any convex polyhedron $G$ in $\RR^d$ is the intersection of multiple $(d-1)$-dimensional hyperplanes. If we move all such hyperplanes inwards (to the direction perpendicular to the hyperplanes) by a small distance $\delta$, they produce another polyhedron, denoted $G_{-\delta}$, called as the $\delta$-contraction of $G$. Note that $G_{-\delta}$ can be empty if $\delta$ is not small enough. 

We prove the result for $d=1$, in which $\bar{\cG}$ is a collection of subintervals in $[0,1]$.  For any subinterval $[a,b]\subset [0,1]$, there exist $x_i, x_j\in \bX_\delta^1$ such that $$x_i\le a\le x_{i+1}, \quad x_j\le b\le x_{j+1}.$$
(By convention, $[x_{i}, x_j]$ is empty if $x_i>x_j$.) Then $([x_i, x_{j+1}], [x_{i+1}, x_{j}])$ is a $2\delta$-bracket of $[a,b]$ since obviously 
\begin{equation}\label{delta:bracket:def}
[x_{i+1}, x_{j}]\subset[a,b]\subset[x_i, x_{j+1}],\,\,\,\,
d_{\triangle}([x_i, x_{j+1}], [x_{i+1}, x_{j}])\le 2\delta.
\end{equation}
There are $\binom{M_\delta+1}{2}$ different choices
of $[x_i, x_j]$, hence, $\binom{M_\delta+1}{2}$ different choices of the pairs $([x_i, x_{j+1}], [x_{i+1}, x_{j}])$. Any $[a,b]\subset [0,1]$ can be $2\delta$ bracketed by one of such pairs in the sense of (\ref{delta:bracket:def}).
This shows that $H_B(2\delta)\le \log\binom{M_\delta+1}{2}\le  2\log(1/\delta)$.
\begin{figure}
    \centering
    \includegraphics[scale=0.35]{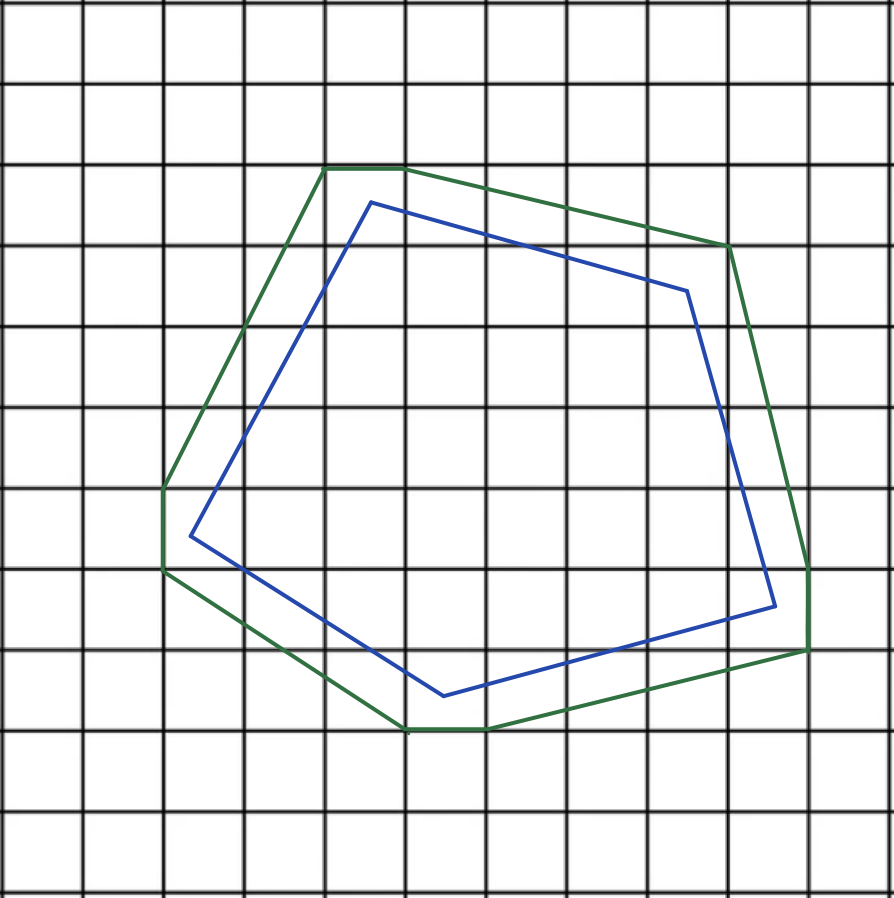}
    \caption{Grid in 2D and the outer cover (green) constructed for with grid points for a polygon (blue).}
    \label{fig:2d}
\end{figure}

When $d\ge 2$, any $G\in\bar{\cG}$ has at most $S$ vertices, so $\bar{G}:= G\cap [0,1]^d$ has at most $dS$ vertices where the factor $d$ is due to the fact that each edge of $G$ intersects at most $d$ edges of $[0,1]^d$ therefore creates at most $dS$ vertices
for $\bar{G}$.
For any polygon $G(\bx_1,\cdots,\bx_s)$ where $s\le dS$, denote $G_{-\sqrt{d}\delta}(\bx_1,\cdots,\bx_s)=G(\bx^{-}_1,\cdots,\bx^-_s)$. 
Each vertex must be in one of the grids in $\bX_\delta^d$. It is easy to see that there exist $\bv_1^1,\ldots,\bv_1^d, \bv_2^1,\ldots,\bv_2^d,\cdots\cdots,\bv_s^1,\ldots,\bv_s^d\in\bX_\delta^d$, where $\bv_i^1,\ldots,\bv_i^d$ are in the same grid,
 such that 
\begin{itemize}
    \item 
    $G(\bx_1,\cdots,\bx_s)\subset G(\bv_1^1,\ldots,\bv_1^d, \bv_2^1,\ldots,\bv_2^d,\cdots\cdots,\bv_s^1,\ldots,\bv_s^d)$;
    \item 
    $\|\bv_i^j-\bx_i\|_2\le \sqrt{d}\delta$ for $i=1,2\cdots, s$ and $j=1,2,\cdots, d$.
\end{itemize}
See Figure \ref{fig:2d} for an illustration when $d=2$.
Similarly for $G(\bx^-_1,\cdots,\bx^-_s)$, there exist
$\bu_1^1,\ldots,\bu_1^d,\bu_2^1,\ldots,\bu_2^d, \cdots\cdots, \bu_s^1,\ldots,\bu_s^d \in\bX_\delta^d$ such that 
\begin{itemize}
    \item 
    $G(\bx^-_1,\cdots,\bx^-_s)\subset G(\bu_1^1,\ldots,\bu_1^d,\bu_2^1,\ldots,\bu_2^d, \cdots\cdots, \bu_s^1,\ldots,\bu_s^d)$;
    \item 
    $\|\bu_i^j-\bx_i^{-}\|_2\le \sqrt{d}\delta$ for $i=1,2\cdots, s$ and $j=1,2,\cdots, d$.
\end{itemize}

By the definition of $G_{-\sqrt{d}\delta}$, we have $\|\bx_i-\bx^-_i\|_2\ge \sqrt{d}\delta$. 
Thus $\|\bu_i^j-\bx^{-}_i\|_2\le \sqrt{d}\delta$ implies $G(\bu_1^1,\ldots,\bu_1^d, \cdots\cdots, \bu_s^1,\ldots,\bu_s^d)\subset G(\bx_1,\cdots,\bx_s)$. 
On the other hand,  
\begin{align*}
&d_\triangle(G(\bu_1^1,\ldots,\bu_1^d, \cdots, \bu_s^1,\ldots,\bu_s^d), G(\bv_1^1,\ldots,\bv_1^d,\cdots,\bv_s^1,\ldots,\bv_s^d))\\
&\le d_\triangle(G_{+\sqrt{d}\delta}(\bx_1,\cdots,\bx_s), G_{-\sqrt{d}\delta}(\bx_1,\cdots,\bx_s))\\
&\le s\cdot 2\sqrt{d}\delta,
\end{align*}
where the term $s$ is due to the fact that $G(\bx_1,\cdots,\bx_s)$ has at most $O(s)$ faces. 
Notice that $$G(\bu_1^1,\ldots,\bu_1^d, \cdots, \bu_s^1,\ldots,\bu_s^d), G(\bv_1^1,\ldots,\bv_1^d,\cdots,\bv_s^1,\ldots,\bv_s^d)\in \bar{\cG}, $$
and $s\le dS$. Thus, with at most $(M_\delta+1)^{d^2S}$ pairs of subsets in $\bar{\cG}$, we can $2d^{3/2}S\delta$-bracket any $\bar{G}\in\bar{\cG}$. 
Therefore,
$$\log \cN_B((2d^{3/2}S\delta),\bar{\cG}, d_\triangle)\lesssim \log\left( (M_\delta+1)^{d^2S}\right),$$
which implies
$$\log \cN_B(\delta,\bar{\cG}, d_\triangle)\lesssim d^2 S\log(d^{3/2}S/\delta).$$



\end{proof}

\begin{lemma}[Theorem 1 in \citep{serra2017bounding}]
\label{lemma:upper_bound_improved}
Consider a deep ReLU network with $L$ layers, $n_l$ ReLU nodes at each layer $l$, and an input of dimension $n_0$. The maximal number of linear pieces of this neural network is at most
\begin{align*}
\sum_{(j_1,\ldots,j_L) \in J} \prod_{l=1}^L \binom{n_l}{j_l},
\end{align*}
where $J = \{(j_1, \ldots, j_L) \in \mathbb{Z}^L: 0 \leq j_l \leq \min\{n_0, n_1 - j_1, \ldots, n_{l-1} - j_{l-1}, n_l\}\ \forall l = 1, \ldots, L\}$. This bound is tight when $L = 1$.
When $n_0 = O(1)$ and all layers have the same width $N$, 
we have 
the same best known asymptotic bound $O(N^{Ln_0})$ first presented in \citep{raghu2017expressive}. 
\end{lemma}

Consider a deep ReLU network with $n_0=d$ inputs and $L$ hidden layers of widths $n_i\geq n_0$ for all $i\in\left[L\right]$. The following lemma establishes a lower bound for the maximal number of linear pieces of deep ReLU networks: 
\begin{lemma}[Theorem 4 in \citep{montufar2014number}] 
    \label{lemma:lb}
    The maximal number of linear pieces of a ReLU network with $n_0$ input units, $L$ hidden layers, and $n_i\geq
    n_0$ rectifiers on the $i$-th layer, is lower bounded by 
    \[
    \left( \prod_{i=1}^{L-1} \left\lfloor \frac{n_i}{n_0}
            \right\rfloor^{n_0} \right)
            \sum_{j=0}^{n_0} { n_L \choose j }. 
    \]
\end{lemma}

\begin{lemma}
\label{lemma:nn_piece}
Let $\cF$ be a class of ReLU neural networks, defined on $\cX=[0,1]^d$, with at most $L$ layers and $N$ neurons per layer. Let $G^f=\{\bx\in\cX: f(\bx)\ge 0\}$ and
$\cG^{\cF} = \{G^f: f\in\cF\}$. Then the bracketing number of $\cG^{\cF}$ satisfies
$$\log \cN_B(\delta,\cG^{\cF}, d_\triangle)\lesssim N^{Ld^2}d^3\left(Ld^2\log( N)\vee \log(1/\delta)\right).$$

\end{lemma}

\begin{proof} 
The proof relies on Lemma \ref{lemma:bracketing}
for which we need to control the number of vertexes of $G^f$ based on the number of pieces (linear regions) of the ReLU neural network. Since ReLU neural networks are piecewise linear, $G^f$ is a collection of sets of polyhedrons. 
Define the subgraph of a function $f:\RR^d\to \RR$ to be the set of points in $\RR^{d+1}$:
$$\textup{sub}(f)=\{(\bx, t): f(\bx)\ge t\}.$$ 
In this sense, 
$\textup{sub}(f)\cap \{(\bx, 0):\bx\in\cX\}=\{(\bx,0): \bx\in G^f\}$, a slice of the subgraph. 
Denote all the pieces to be $p_1,p_2,\cdots,p_s$. Each piece is a $d$-dimensional polyhedron on which $f(x)$ is linear. To control the complexity of $G^f$, consider the most extreme case that 
the function crosses zero on each piece, i.e. for any $i=1,\ldots,s$,
$\{(\bx,f(\bx)):\bx\in p_i\}\cap\{(\bx,0):\bx\in \cX\}\neq\emptyset$.
Each intersection resides in a $(d-1)$-dimensional hyperplane, e.g. dot for $d=1$, line segment for $d=2$ and so on. So the number of such $(d-1)$-dimensional hyperplanes in $G^f$ is at most $s$. 

A vertex of a polyhedron in $[0,1]^d$ can be thought of as the intersection of at least $d$ hyperplanes of dimension $d-1$. Thus, with at most $s$ hyperplanes there are at most $\binom{s}{d}<s^d$
vertices in $G^f$. 
In order to apply Lemma \ref{lemma:bracketing}, we break the collection of polyhedrons into the so-called {\em basic polyhedrons} each with $d+1$ vertices. For instance, the basic polyhedrons are intervals when $d=2$, are triangles when $d=3$, and so on. 

\begin{figure}
    \centering
    \includegraphics[scale=0.28]{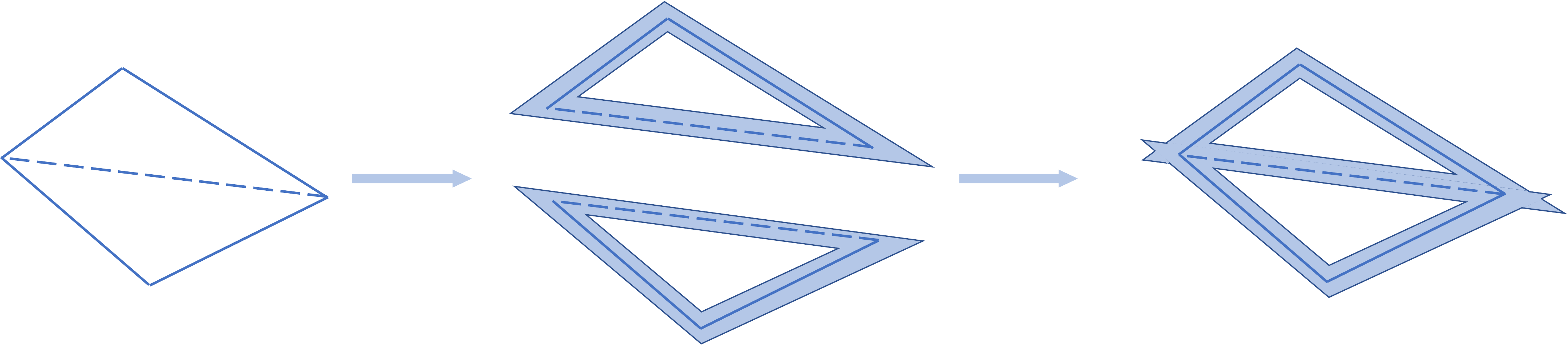}
    \caption{Demonstration of how a polygon in $d=2$ case can be divided into basic triangles. The union of the two brackets form a bracket of the original polygon. The blue shade is the symmetric difference.}
    \label{fig:divide}
\end{figure}

A polyhedron $G$ with at most $s$ vertices can be divided into at most $s$ disjoint basic polyhedrons $B_1,\ldots,B_s$. For instance, Figure \ref{fig:divide} demonstrates the $d=2$ case. Therefore, the bracketing number of the polyhedrons can be derived by bracketing the basic polyhedrons.
For a basic polyhedron $B$, denote its $\delta$-bracketing pair to be $(U_{B,\delta},V_{B,\delta})$, i.e., $U_{B,\delta}\subset B\subset V_{B,\delta}$.
Then $ (U_{G, \delta}, V_{G, \delta})$, defined as below
\begin{align*}
    U_{G,\delta} =& U_{B_1,\delta}\cup U_{B_2,\delta}\cup\cdots\cup U_{B_s,\delta}\\
    V_{G, \delta} =& V_{B_1,\delta}\cup V_{B_2,\delta}\cup\cdots\cup V_{B_s,\delta},
\end{align*}
form a $(s\delta)$-bracket of $G$. 
Hence, the bracketing number of all polyhedrons is controlled by
the $s$-th power of the bracketing number of all basic polyhedrons.
Applying Lemma \ref{lemma:upper_bound_improved} we know 
$s = O(N^{Ld})$ and the number of vertices is at most $S =O(N^{Ld^2})$. Together with 
Lemma \ref{lemma:bracketing},
we therefore get that
\[
\log \cN_B(S\delta,\cG^{\cF}, d_\triangle)\lesssim S(d+1)d^2\log((d+1)d^{3/2}/\delta),
\]
which implies
\begin{eqnarray*}
    \log \cN_B(\delta,\cG^{\cF}, d_\triangle)&\lesssim& N^{Ld^2}d^3\log (N^{Ld^2}d^{3}/\delta)\\
    &\lesssim& N^{Ld^2}d^3\left(Ld^2\log( N)\vee \log(1/\delta)\right).
\end{eqnarray*}
\end{proof}

More discussions about Lemma \ref{lemma:nn_piece} can be found in Appendix \ref{app:entropy}. 
Next, we present some lemmas that can take advantage of the obtained entropy bound and eventually take us to the proof of the excess risk convergence rate.
\begin{lemma}
[Theorem 5.11 in \cite{sara2000}]
\label{lemma:511}
For some function space $\cH$ with $\sup_{h\in\cH}\|h(\bx)\|_\infty\le K$ and $\sup_{h\in\cH}\|h(\bx)\|_{L_2(P)}\le R$ where $P$ is the distribution of $\bx$.
Take $a>0$ satisfying (1) $a\le C_1\sqrt{n}R^2/K$; (2) $a\le 8\sqrt{n}R$;
\[
(3)\quad a\ge C_0\left(\int_{a/64\sqrt{n}}^R H_B^{1/2}(u,\cF,L_2(P))du\vee R\right);
\]
and (4) $C_0^2\ge C^2(C_1+1)$. Then 
\[
\PP\left(\sup_{h\in\cH}\abr{\sqrt{n}\int hd(P_n-P)}\ge a\right)\le C\exp\rbr{-\frac{a^2}{C^2(C_1+1)R^2}},
\]
where $P_n$ is the empirical counterpart of $P$.
\end{lemma}

So far, the presented lemmas are only concerned with the general case, i.e. set $G^*$, $p,q$, etc. that does not depend on $n$. However, in our teacher-student framework, the optimal set $G^*_n$ is indexed by $n$ as it's determined by the teacher network $\cF^*_n$. In the remaining part of the proof, we will consider specifically for our teacher network case.  
The next lemma investigates the modulus of continuity of the empirical process. It's similar to Lemma 5.13 in \cite{sara2000} but with a key difference in the entropy assumption (\ref{eqn:hb}), where the entropy bound contains $n$. 
\begin{lemma}
\label{lemma:geer}
For a probability measure $P$, let $\cH_n$ be a class of uniformly bounded (by 1) functions $h$ in $L_2(P)$ depending on $n$. Suppose that the $\delta$-entropy with bracketing $H_B(\delta,\cH_n,L_2(P))$ satisfies, for some $A_n>0$, the inequality
\begin{align}
\label{eqn:hb}
    H_B(\delta,\cH_n, L_2(P))\le A_n\log (1/\delta)
\end{align}
for all $\delta>0$ small enough. Let $h_{n0}$ be a fixed element in $\cH_n$.
Let $\cH_n(\delta) = \{h_n\in\cH_n: \|h_n-h_{n0}\|_{L_2(P)}\le \delta\}$.
Then there exist constants $D_1>0, D_2>0$ such that for a sequence of i.i.d. random variables $\bx_1,\cdots,\bx_n$ with probability distribution $P$, it holds that 
\begin{align*}
   & \PP\left(\sup_{\substack{h_n\in\cH_n;\\ \|h_n-h_{n0}\|> \sqrt{\frac{A_n}{n}}}}\frac{\sum_{i=1}^n[ (h_n - h_{n0})(\bx_i) - \EE(h_n-h_{n0})(\bx_i)]}{\sqrt{\frac{A_n}{n}}\|h_n-h_{n0}\|_2\log(\|h_n-h_{n0}\|_2^{-1})}>D_1 x\right)\\
    &\le D_2 e^{-A_n x}
\end{align*}
for all $x\ge 1$. 
\end{lemma}
\begin{proof}
The main tool for the proof is Lemma \ref{lemma:511}. Replace $\cH$ with $\cH_n(\delta)$ in Lemma \ref{lemma:511} and take $K=4$, $R = \sqrt{2}\delta$ and $a = \frac{1}{2}C_1\sqrt{A_n}\delta\log(1/\delta)$, with $C_1 = 2\sqrt{2}C_0$. Then (1) is satisfied if 
\begin{equation}
\label{eqn:an}
    \frac{\sqrt{A_n}}{\delta}\log(\frac{1}{\delta})\le \sqrt{n}.
\end{equation} 
Under (\ref{eqn:an}), condition (2) and (3) are satisfied automatically. Choosing $C_0$ sufficiently large will ensure (4). Thus, for all $\delta$ satisfying (\ref{eqn:an}), we have
\begin{align*}
    &\PP\left(\sup_{h_n\in\cH_n(\delta)}\abr{\sqrt{n}\int (h_n-h_{n0})d(P_n-P)}\ge \frac{C_1}{2}\sqrt{A_n}\delta\log(1/\delta)\right)\\
    &\le C\exp\left(-\frac{C_1A_n\log^2(1/\delta)}{16C^2}\right)
\end{align*}
Notice that $(\ref{eqn:an})$ holds if $\delta\ge \sqrt{A_n/n}$.
Let $B = \min\{b>1: 2^{-b}\le \sqrt{A_n/n}\}$ and apply the peeling device. Then,
\begin{align*}
    &\PP\rbr{\sup_{\substack{h_n\in\cH_n;\\ \|h_n-h_{n0}\|> \sqrt{A_n/n}}}\frac{\abr{\sqrt{n}\int (h_n-h_{n0})d(P_n-P)}}{\sqrt{A_n}\|h_n-h_{n0}\|\log(1/\|h_n-h_{n0}\|)}\ge \frac{C_2}{2}}\\
    &\le\sum_{b=1}^B\PP\rbr{\sup_{h_n\in\cH_n(2^{-b})}{\abr{\sqrt{n}\int (h_n-h_{n0})d(P_n-P)}}\ge \frac{C_1}{2}\sqrt{A_n}2^{-b}\log(2^{b}) }\\
    &\le\sum_{b=1}^B C\exp\left(-\frac{C_1A_nb^2\log^2(2)}{16C^2}\right)\le 2C\exp\rbr{-\frac{C_1A_n}{16C^2}},
\end{align*}
if $C_1A_n$ is sufficiently large.
\end{proof}


We then present a lemma that establishes the connection between $d_{\triangle}$ and $d_{p,q}$, which is adapted from Lemma 2 in \cite{mammen1999smooth} to our teacher network setting. Corresponding to assumption (A3), we define (N$_n$) as an extension to the classical Tsybakov noise condition (N). 
\begin{itemize}
    \item[(N$_n$)] There exists $c_n>0$ depending on $n$ and $\kappa\in[0,\infty]$ such that for any $0\le t\le T_n$
    $$\PP\left(\{\bx:|p_n(\bx)-q_n(\bx)|\le t\}\right)\le c_n t^\kappa.$$
\end{itemize}
Note that the (N) is a special case of (N$_n$) with $T_n$ and $c_n$ being absolute constant.  
\begin{lemma}
\label{lemma:distance_n}
Assume (N$_n$) and $p_n,q_n$ are bounded by $b_2>0$. Then, there exists absolute constants $b_1(\kappa)>0$ depending on $\kappa$ such that for 
any Lebesgue measurable subsets $G_1$ and $G_2$ of $\cX$,
\begin{align*}
    b_1(\kappa)\left(T_n\wedge c_n^{-1/\kappa}\right)d_{\triangle}^{(\kappa+1)/\kappa}(G_1, G_2)\le d_{p_n, q_n}(G_1, G_2)\le 2b_2 d_{\triangle}(G_1, G_2).
\end{align*}
\end{lemma}
\begin{proof}
The second inequality is trivial given that $p,q$ are bounded by $b_2$. For the first inequality, since $\QQ(|p_n-q_n|\le t)\le c_n t^\kappa$ for all $0\le t\le T_n$, the boundedness of $\QQ(\cX)$ implies that 
$$\QQ(|p_n-q_n|\le t)\le A_n t^\kappa, \ \ \forall \   t>0,$$ 
where $A_n=\left(\frac{\QQ(\cX)}{T_n^{\kappa}} \vee c_n\right).$ 
Then,
\begin{align*}
   & d_{p_n,q_n}(G_1,G_2)\\
    &\ge \int_{G_1\triangle G_2}|p_n-q_n|\II\{|p_n-q_n|\ge \left(\frac{d_\triangle(G_1,G_2)}{2A_n}\right)^{1/\kappa}\}d\QQ\\
    &\ge \left(\frac{d_\triangle(G_1,G_2)}{2A_n}\right)^{1/\kappa}\left[\QQ(G_1\triangle G_2)-\QQ(|p_n-q_n|< \left(\frac{d_\triangle(G_1,G_2)}{2A_n}\right)^{1/\kappa})\right]\\
    &\ge \frac{d_\triangle(G_1,G_2)^{1+1/\kappa}}{(2A_n)^{1/\kappa}} - 1/2\frac{d_\triangle(G_1,G_2)^{(\kappa+1)/\kappa}}{(2A_n)^{1/\kappa}}\\
    &\ge \frac{2^{-(\kappa+1)/\kappa}}{A_n^{1/\kappa}} d_\triangle(G_1,G_2)^{(\kappa+1)/\kappa}.
\end{align*}
\end{proof}

Our goal in classification is to estimate $G^*_n$ by
$\hat{G}_{n} = \argmin_{G\in\cG_n} R_{n}(G)$, where $\cG_n$ is some collection of sets associated with the student network $\cF_n$ and
$$R_{n}(G) = \frac{1}{2n}\sum_{i=1}^n\left( \II\{\bx_i\in G |y_i=1\}(\bx)+
\II\{\bx_i\notin G |y_i=-1\}(\bx)\right).$$ 
Similar to Theorem 1 in \cite{mammen1999smooth}, we have the following lemma regarding the upper bound on the rate of convergence. 
\begin{lemma}
\label{lemma:thm1}
Suppose $0<\QQ(\cX)<\infty$
and let $\cG_n^*$ be a collection of subsets of $\cX\subset \RR^d$. Define  
\begin{align}
\label{eqn:fg}
\begin{split}
    \cD^{\cG_n^*}_n = \{&(p_n,q_n): \QQ\{\bx\in \cX: |p_n(\bx)-q_n(\bx)|\le t\}\le c_n t^\kappa \mbox{ for } 0\le t\le T_n,\\
    &\{\bx\in \cX: p_n(\bx)\ge q_n(\bx)\}\in\cG_n^*, p_n(\bx), q_n(\bx)\le b_2 \mbox{ for } x\in \cX
    \},
    \end{split}
\end{align} 
where $b_2$ is an absolute constant. Let $\cG_n$ be another class of subsets satisfying $\cG_n^*\subset\cG_n$. Suppose there exist positive constants $A_n>0$ depending on $n$ such that for any $\delta>0$ small enough, 
\begin{equation}
\label{entropy}
H_B(\delta, \cG_n, d_{\triangle})\le  A_n\log(1/\delta).
\end{equation} 
Then we have
\begin{equation}
\label{rate}
    \lim_{n\to\infty} \sup_{(p_n,q_n)\in \cD_n^{\cG_n^*}} \rbr{\frac{A_n\log^2 n}{n}}^{-\frac{\kappa+1}{\kappa+2}}\left(T_n\wedge c_n^{-1/\kappa}\right)^{\frac{\kappa}{\kappa+2}}\EE [d_{p_n,q_n}(\hat{G}_{n}, G_n^*)] <\infty.
\end{equation}
\end{lemma}

\begin{proof}
\label{thm1}
For $(p_n,q_n)\in \cF_n^{G^*_n}$, let $G^*_n=\{\bx\in\cX: p_n(\bx)\ge q_n(\bx)\}$. For a given set $G\in\cX$, let
$h_G(\bx)=\II\{\bx\in G\}$. In particular, let $h_n^*=h_{G^*_n}$.
Let $\|h\|_p^2 = \int h^2(\bx)p(\bx)\QQ (d\bx)$. Since both $p_n$ and $q_n$ are bounded, 
\begin{align}
    \begin{split}
        \|h_{G_n}-h^*_n\|_p^2 &= \int_{G_n\triangle G^*_n}p_n(\bx)\QQ(d\bx)\le b_2 d_{\triangle}(G_n,G^*_n),\\
        \|h_{G_n}-h^*_n\|_q^2 &= \int_{G_n\triangle G^*_n}q_n(\bx)\QQ(d\bx)\le b_2 d_{\triangle}(G_n,G^*_n).
    \end{split}
\end{align}
Consider the random variable 
\[
V_{n} = -\sqrt{n}\ \frac{R_n(\hat{G}_n)-R_n(G^*_n)-\EE(R_n(\hat{G}_n)-R_n(G^*_n))}{\sqrt{A_nd_{\triangle}(G^*_n,\hat{G}_n)}\log(1/d_{\triangle}(G^*_n,\hat{G}_n))}.
\]
Since $\cG_n^*\subset\cG_n$, we have $R_n(\hat{G}_n)\le R_n(G^*_n)$. Thus 
\begin{equation}
    \label{eqn:2v}
 \frac{   \sqrt{n}\EE(R_n(\hat{G}_n)-R_n(G^*_n))}{\sqrt{A_nd_{\triangle}(G^*_n,\hat{G}_n)}\log(1/d_{\triangle}(G^*_n,\hat{G}_n))} \le V_n.
\end{equation}

Note that
\begin{align*}
    R_n(G_n)-R_n(G^*_n) =& \frac{1}{2n}\sum_{i=1}^n\II_{\{y_i=1\}}(h^*_n-h_{G_n})(\bx_i)\\
    &+\frac{1}{2n}\sum_{i=1}^n\II_{\{y_i=-1\}}(h_{G_n}-h^*_n)(\bx_i).
\end{align*}
Then $V_n$ can be written as 
\begin{align*}
V_{n} =& \ \frac{(1/2n)\sum_{i=1}^n\II_{\{y_i=1\}}(h_{\hat{G}_n}-h^*_n)(\bx_i) - \EE(\II_{\{y=1\}}(h_{\hat{G}_n}-h^*_n)(\bx))}{\sqrt{A_nd_{\triangle}(G^*_n,\hat{G}_n)/n}\log(1/d_{\triangle}(G^*_n,\hat{G}_n))}+\\
& \frac{(1/2n)\sum_{i=1}^n\II_{\{y_i=-1\}}(h^*_n-h_{\hat{G}_n})(\bx_i) - \EE(\II_{\{y=-1\}}(h^*_n-h_{\hat{G}_n})(\bx))}{\sqrt{A_nd_{\triangle}(G^*_n,\hat{G}_n)/n}\log(1/d_{\triangle}(G^*_n,\hat{G}_n))}.    
\end{align*}

Consider the event 
$E_n = \{d_{\triangle}(G^*_n,\hat{G}_n)>\sqrt{A_n/n}\}$
and let $\tilde{\cG}_n=\{G\in\cG_n: d_{\triangle}(G,G^*_n)>\sqrt{A_n/n}\}$.
If $E_n$ holds, then
\begin{align*}
    V_n =&-\sqrt{n}\ \frac{R_n(\hat{G}_n)-R_n(G^*_n)-\EE(R_n(\hat{G}_n)-R_n(G^*_n))}{\sqrt{A_nd_{\triangle}(G^*_n,\hat{G}_n)}\log(1/d_{\triangle}(G^*_n,\hat{G}_n))}\\
    \le& \sup_{G_n\in\tilde{\cG}_n }\sqrt{n}\ \frac{R_n(G^*_n)-R_n(G_n)-\EE(R_n(G_n)-R_n(G^*_n))}{\sqrt{A_nd_{\triangle}(G^*_n,\hat{G}_n)}\log(1/d_{\triangle}(G^*_n,\hat{G}_n))}\\
    \le& \sup_{G_n\in\tilde{\cG}_n }\ \frac{|(1/2n)\sum_{i=1}^n\II_{\{y_i=1\}}(h_{{G}_n}-h^*_n)(\bx_i) - \EE(\II_{\{y=1\}}(h_{{G}_n}-h^*_n)(\bx))|}{\sqrt{A_nd_{\triangle}(G^*_n,\hat{G}_n)/n}\log(1/d_{\triangle}(G^*_n,\hat{G}_n))}+\\
    & \sup_{G_n\in\tilde{\cG}_n }\ \frac{|(1/2n)\sum_{i=1}^n\II_{\{y_i=-1\}}(h_{{G}_n}-h^*_n)(\bx_i) - \EE(\II_{\{y=-1\}}(h_{{G}_n}-h^*_n)(\bx))|}{\sqrt{A_nd_{\triangle}(G^*_n,\hat{G}_n)/n}\log(1/d_{\triangle}(G^*_n,\hat{G}_n))}\\
    \le& \sup_{h_{n}\in\cH_n }\ \frac{|(1/2n)\sum_{i=1}^n\II_{\{y_i=1\}}(h_{n}-h^*_n)(\bx_i) - \EE(\II_{\{y=1\}}(h_{n}-h^*_n)(\bx))|}{2b_2^{-1/2}\sqrt{A_n/n}\|h_{n}-h^*_n\|_p\log(\sqrt{b_2}/\|h_{n}-h^*_n\|_p)}+ \\
    & \sup_{h_{n}\in\cH_n }\ \frac{|(1/2n)\sum_{i=1}^n\II_{\{y_i=1\}}(h_{n}-h^*_n)(\bx_i) - \EE(\II_{\{y=1\}}(h_{n}-h^*_n)(\bx))|}{2b_2^{-1/2}\sqrt{A_n/n}\|h_{n}-h^*_n\|_q\log(\sqrt{b_2}/\|h_{n}-h^*_n\|_q)},
\end{align*}
where $\cH_n = \{h_n(\bx)=\II{\{\bx\in G_n\}}: G_n\in \cG_n\}$. The last inequality follow from the fact that $\sqrt{x}\log(1/x)$ is strictly increasing when $x<1$.
Notice that $h_n$'s are uniformly bounded by 1 and the $L_2$ norm squared of $h_{G_1}-h_{G_2}$ is $d_\triangle(G_1,G_2)$. 
Applying Lemma \ref{lemma:geer}, we have 
\begin{equation}
\label{eqn:EVn}
    \EE[V_n\II(E_n)]\le C
\end{equation}
for some finite constant $C$. Now we use this inequality to prove the main result.
From (\ref{eqn:2v}), we know that 
$$ d_{p_n,q_n}(\hat{G}_n, G^*_n)\le V_n (A_n/n)^{1/2}d_{\triangle}(G^*_n,\hat{G}_n)^{1/2}\log(1/d_{\triangle}(G^*_n,\hat{G}_n)),$$
which, together with Lemma \ref{lemma:distance_n}, yields that
\begin{align*}
    d_{p_n,q_n}(\hat{G}_n, G^*_n)\lesssim &  V_n (A_n/n)^{1/2} \left(T_n\wedge c_n^{-1/\kappa}\right)^{-\frac{\kappa}{2(\kappa+1)}} d_{p_n,q_n}(\hat{G}_n, G^*_n)^{\frac{\kappa}{2(\kappa+1)}}\\
    & \cdot [\log(1/d_{p_n,q_n}(\hat{G}_n, G^*_n))+\log(b_1(\kappa)(T_n\wedge c_n^{-1/\kappa}))],
\end{align*}
which simplifies to be
\[
d_{p_n,q_n}(\hat{G}_n, G^*_n)\lesssim V_n^{\frac{2\kappa+2}{\kappa+2}} \rbr{\frac{A_n\log^2 n}{n}}^{\frac{\kappa+1}{\kappa+2}}\left(T_n\wedge c_n^{-1/\kappa}\right)^{-\frac{\kappa}{\kappa+2}}.
\]
where we used the fact that $d_{p_n,q_n}(\hat{G}_n, G^*_n)\gtrsim 1/n$.
Therefore, under $E_n$, (\ref{eqn:EVn}) implies that  
\begin{align*}
\EE[d_{p_n,q_n}(\hat{G}_n, G^*_n)]&\lesssim
\rbr{\frac{A_n\log^2 n}{n}}^{\frac{\kappa+1}{\kappa+2}}\left(T_n\wedge c_n^{-1/\kappa}\right)^{-\frac{\kappa}{\kappa+2}}.
\end{align*}
On the other hand, under $E_n^c$, we have $$d_{\triangle}(\hat{G}_n,G^*_n)\le \sqrt{A_n/n}.$$
By Lemma \ref{lemma:distance_n} we know $d_{p,q}(\hat{G}_n, G^*_n)$ is also bounded by $\sqrt{A_n/n}$. Since $(\kappa+1)/(\kappa+2)\le 1$, the rate under $E_n$ dominates and the proof is complete.

\end{proof}

\begin{proof}[Proof of Theorem \ref{thm:ts}]

First, we verify that the Tsybakov noise condition holds for $\kappa = 1$ in our setting. The proof is based on the fact that a ReLU network is piecewise linear and the number of linear pieces is quantifiable.
Assumption (A3)
implies (N$_n$) with $c_n, 1/T_n=O(\log n)^{m^* d^2 L_n^*}$ and $\kappa=1$.
In the case where $p,q$ have disjoint support, obviously $\kappa$ can be arbitrarily large.

Next, we consider the bracketing number of $\cG_n$ defined via $\cF_n$ that $\cG_n=\{\bx\in\cX: f(\bx)\ge 0, f\in\cF_n\}$.
From Lemma \ref{lemma:nn_piece} we have 
\begin{align*}
     \log \cN_B(\delta,\cG_n, d_\triangle)
    &\lesssim N^{Ld^2}d^2\left(Ld^2\log( N)\vee \log(1/\delta)\right).
\end{align*}
Thus, $A_n = O(N_n)^{d^2 L_n}$ as in (\ref{entropy}) if $\delta\ll 1/N$.
Recall from assumption (A2) and (A3) that $L_n=O(1)$, $N_n=O(\log n)^{m}$ and $1/T_n, c_n=O(\log n)^{m^*d^2L_n^*}$. 
Applying Lemma \ref{lemma:thm1} with $\kappa=1$ we have that the excess risk has upper bound
\begin{align*}
&\sup_{(p, q)\in\tilde{\cF}^*_n}\EE[\cE(\hat{f}_n,C^*_n)]\\
&\lesssim \rbr{\frac{A_n\log^2 n}{n}}^{\frac{2}{3}}\left(T_n^{-1}\wedge c_n\right)^{\frac{1}{3}}\\
&\lesssim \rbr{\frac{1}{n}}^{\frac{2}{3}}\left(\log n\right)^{\frac{2}{3}(md^2L_n+2)+\frac{1}{3}m^*d^2L_n^*}.
\end{align*}
\end{proof}

\begin{proof}[Proof of Corollary \ref{cor:dis}]
Corollary \ref{cor:dis} easily follows from the fact that $p,q$ having disjoint support implies $\kappa=\infty$ in (N$_n$). 
\end{proof}

\subsection{Proof of Theorem \ref{thm:lb}}
We will show that the lower bound holds in special case that (1) assumption (A3) satisfies $c_n, 1/T_n$ being absolute constants that doesn't depend on $n$; (2) instead of general ReLU neural network $f_n^*\in\cF^*$, we consider a special structure where $f_n^*$ is linear in one of the dimensions, reminiscent of the ``boundary fragment" assumption. In this special case, we are able to show the best possible convergence rate already matches that in Theorem \ref{thm:ts}.
For ease of notation, we omit the subscript $n$ and write $p_n,q_n$ as $p,q$ if no confusion arises.
\begin{proof}
Without loss of generality, let $\cX=[0,1]^d$. Consider the ``boundary fragment'' setting and let $\tilde{\cG}_n$ be a set defined by a ReLU network family $\tilde{\cF}_n$ containing functions from $\RR^{d-1}$ to $\RR$:
$$\tilde{\cG}_n =\{(x_1,\cdots,x_d)\in [0,1]^d: 0\le x_j\le h(\bx_{-j}), h\in\tilde{\cF}_n, j=1,\cdots,d\},$$ 
where $\bx_{-j} = (x_1,\cdots,x_{j-1}, x_j,\cdots,x_d)$. 
Notice that if $h(\bx_{-j})$ is a ReLU network on $\RR^{d-1}$, then $\tilde{h}(\bx)= h(\bx_{-j})-x_j$ is a ReLU network on $\RR^d$. Thus $\tilde{\cG}_n$ is a subset of $\cG_n$ which corresponds to the student network that
\begin{equation}
\label{eqn:cgn}
    \cG_n = \{\bx\in \cX: f(\bx)>0, f\in\cF_n\}
\end{equation}

Let $\tilde{G}_n$ denote the empirical 0-1 loss minimizer over $\tilde{\cG}_n$. 
To show the lower bound, consider the subset of $\cD^{\tilde{\cG}_n}$ (\ref{eqn:fg}) that contains all pairs like $(p,q_0)$, where $p\in \cF_1,q_0$  will be specified later. Then,
\begin{align*}
    \sup_{(p,q)\in \cD^{\tilde{\cG}_n}}\EE d_{\triangle}(\tilde{G}_{n}, G^*)&\ge\sup_{(p,q_0):p\in\cF_1} \EE d_{\triangle}(\tilde{G}_{n}, G^*)\\
    &\ge \EE_{q_0}\left[\frac{1}{|\cF_1|}\sum_{p\in\cF_1}\EE_p[d_\triangle(\tilde{G}_{n}, G^*)|\cD_{q_0}]\right],
\end{align*}
where $\cF_1$ is a finite set to be specified later, $p, q_0$ are the underlying densities for the two labels and $\cD_{q_0}$ denotes all the data generated from $q_0$.

For ease of presentation, we first give the proof for the case $d=2$ and then extend to general $d$. 
Let $\phi(t)$ be a piecewise linear function supported on $[-1,1]$ defined as
\[
\phi(t) = \begin{cases}
t+1 & -1< t\le 0,\\
-t+1 & 0< t< 1,\\
0 & |t|\ge 1.
\end{cases}
\]
Rewrite $\phi$ as
$\phi(t)=\sigma(t+1)-\sigma(t) + \sigma(-t+1) - \sigma(-t)-2$,
which is a one hidden layer ReLU neural network with 11 non-zero weights that are either $1$ or $-1$.
For $\bx=(x_1, x_2)\in [0,1]^2$, define
\begin{align*}
    q_0(\bx)= &(1-\eta_0-b_1)\II\{0\le x_2<1/2\} + \II\{1/2\le x_2<1/2+ e^{-M}\} \\
    &+(1+\eta_0+b_2)\II\{1/2+ e^{-M}\le x_2\le 1\},
\end{align*}
where $M\ge 2$ is an integer to be specified later. Let $b_1 = c_2^{-1/\kappa}e^{-M/\kappa}$ and $b_2>0$ be chosen such that $q_0$ integrates to 1 (so $q_0$ is a valid probability density).

For $j=1,2,\cdots, M$ and $t\in[0,1]$, let 
$$\psi_j(t) =  e^{-M}\phi\left(M\left[t-\frac{j-1}{M}\right]\right).$$
Note that $\psi_j$ is only supported on $[\frac{j-1}{M}, \frac{j}{M}]$. 
For any vector $\omega = (\omega_1,\cdots,\omega_M)\in\Omega:=\{0,1\}^M$, define
\begin{equation*}
    b_\omega(t) = \sum_{j=1}^M \omega_j\psi_j(t),
\end{equation*}
and
\begin{align*}
    p_\omega(\bx) =& 1+\left[\frac{1/2+  e^{-M}-x_2}{c_2}\right]^{1/\kappa}\II\{1/2\le x_2\le1/2+b_\omega(x_1)\}\\
    &-b_3(\omega)\II\{1/2+b_\omega(x_1)< x_2\le 1\},
\end{align*}
where $b_3(\omega)>0$ is a constant depending on $\omega$ chosen such that $p_\omega(x)$ integrates to 1. Let $\cF_1 = \{p_\omega: \omega\in \Omega\}$ and we will show that $(p_\omega, q_0)\in\cD^{\tilde{\cG}_n}$ for all $\omega\in\Omega$.

To this end, we need to verify that 
\begin{enumerate}[(a)]
    \item 
    $p_\omega(\bx)\le c_1$ for $\bx\in [0,1]^2$; 
    \item
    $\{\bx\in \cX: p_\omega(\bx)\ge q_0(\bx)\}\in \cG_n$;
    \item
    $\QQ\{\bx\in \cX: |p_\omega(\bx)-q_0(\bx)|\le\eta\}\le c_2\eta^{\kappa}$.
\end{enumerate} 
For (a), since $p_\omega$ integrates to 1, 
\begin{align*}
    b_3(\omega)&\le \max_{\{1/2\le x_2\le1/2+b_\omega(x_1)\}} \left[\frac{1/2+ e^{-M}-x_2}{c_2}\right]^{1/\kappa}=O(e^{-M/\kappa}).
\end{align*}
Thus, $p_\omega(\bx)\le c_1$ for a large enough $M$ and some absolute constant $c_1$. For (b), notice that
\begin{align*}
\{\bx: p_\omega(\bx)\ge q_0(\bx)\} 
&= \{\bx: 0\le x_2\le 1/2+b_\omega(x_1)\}\\
&= \{\bx\in [0,1]^2:  b_\omega(x_1)-\sigma(x_2)+1/2\ge 0\}\in\cG_n,    
\end{align*}
where the last inclusion follows from the definition of $\cG_n$ (\ref{eqn:cgn}) and the fact that $b_\omega(x_1)-\sigma(x_2)+1/2$ is a ReLU neural network with one hidden layer, whose width and number of non-zero weights are both $O(M)$. Later we will see that $M = O(\log n)$, and thus the constructed neural network satisfies all the size constraints in Theorem \ref{thm:ts}. For (c), it follows that
\begin{align*}
    &Q\{\bx\in \cX: |p_\omega(\bx)-q_0(\bx)|\le\eta\}\\
    \le& Q\{\bx\in \cX: 1/2\le x_2\le 1/2+ e^{-M}, \left[\frac{1/2+e^{-M}-x_2}{c_2}\right]^{1/\kappa}\le\eta\}\\
    \le &Q \{\bx\in \cX: 1/2+e^{-M}-c_2\eta^\kappa \le x_2 \le 1/2+e^{-M}\}\\
     \le & c_2\eta^\kappa.
\end{align*}

Since the above (a)-(c) hold and by the definition of $\cD^{\tilde{\cG}_n}$ (\ref{eqn:fg}), we conclude that $(p_\omega, q_0)\in\cD^{\tilde{\cG}_n}$ for all $\omega\in\Omega$ .
We next establish how fast $S:=|\cF_1|^{-1}\sum_{p\in\cF_1}\EE_p[d_\triangle(\tilde{G}_{n}, G^*)|\cD_{q_0}]$ can converge to zero. To this end, we use the Assouad's lemma stated in \citep{korostelev2012minimax}
which is adapted to the estimation of sets. 

For $j=1, \cdots, M$ and $\omega=(\omega_1,\cdots,\omega_M)\in\Omega$, let 
\begin{align*}
   & \omega_{j0} = (\omega_1,\cdots, \omega_{j-1}, 0 , \omega_{j+1},\cdots, \omega_M)\\
   &   \omega_{j1} = (\omega_1,\cdots, \omega_{j-1}, 1 , \omega_{j+1},\cdots, \omega_M)
\end{align*}
For $i=0$ and $i=1$, let $P_{ji}$ be the probability measure corresponding to the distribution of $x_1,\cdots, x_n$ when the underlying density is $f_{\omega_{ji}}$. Denote the expectation w.r.t. $P_{ji}$ as $\EE_{ji}$. Let 
\begin{align*}
    \cD_j &= \{\bx\in \cX: 1/2+b_{\omega_{j0}}(x_1)<x_2\le1/2+b_{\omega_{j1}}(x_1)\}\\
    & = \{\bx\in \cX:b_{\omega_{j0}}(x_1)<x_2-1/2\le b_{\omega_{j0}}(x_1)+ \psi_j(x_1)\}.
\end{align*}
Then 
\begin{align*}
S&\ge1/2\sum_{j=1}^M \QQ(\cD_j)\int\min\{dP_{j1},dP_{j0}\}\\
&\ge 1/2\sum_{j=1}^M\int_0^1 \psi_j(x_1)dx_1\int\min\{dP_{j1},dP_{j0}\}\\
&\ge 1/2\sum_{j=1}^M  e^{-M}\int\phi(Mt)dt\int\min\{dP_{j1},dP_{j0}\}\\
&\ge \frac{1}{4}\sum_{j=1}^M  e^{-M}\int\phi(Mt)dt\left[1-H^2(P_{10},P_{11})/2\right]^n,
\end{align*}
where $H(\cdot,\cdot)$ denotes the Hellinger distance. 
Then it holds that
\begin{align*}
H^2(P_{10},P_{11}) =& \int \left[\sqrt{f_{\omega_{10}}(\bx)}-\sqrt{f_{\omega_{11}}(\bx)}\right]^2d\bx\\
\le&\int_{0}^1\Bigg\{\int_{1/2}^{1/2+\psi_1(x_1)}\left[1-\sqrt{1+\left(\frac{1/2+ e^{-M}-x_2}{c_2}\right)^{1/\kappa}}\right]^2dx_2 \\
&+\int_{1/2}^1\left[\sqrt{1-b_3(\omega_{10})}-\sqrt{1-b_3(\omega_{11})}\right]^2dx_2\Bigg\}dx_1\\
\le&\int_{0}^1\int_{e^{-M}-\psi_1(x_1)}^{ e^{-M}}\left[1-\sqrt{1+\left(\frac{v}{c_2}\right)^{1/\kappa}}\right]^2dvdx_1 \\
&+|b_3(\omega_{10})-b_3(\omega_{11})|^2.
\end{align*}

We will analyze the last two terms.
For the first term,
\begin{align*}
&\int_{0}^1\int_{e^{-M}-\psi_1(x_1)}^{ e^{-M}}\left[1-\sqrt{1+\left(\frac{v}{c_2}\right)^{1/\kappa}}\right]^2dvdx_1\\
&\le \int_{0}^1\int_{ e^{-M}-\psi_1(x_1)}^{ e^{-M}}\left(\frac{v}{c_2}\right)^{2/\kappa}dvdx_1 \\
&\le\frac{\kappa c_2^{-2/\kappa}}{\kappa+2}\int_0^1\left(e^{-M}\right)^{1+2/\kappa}- \left( e^{-M}-\psi_1(x_1)\right)^{1+2/\kappa}dx_1\\
&\le\frac{\kappa c_2^{-2/\kappa}}{\kappa+2}\left( e^{-M}\right)^{1+2/\kappa}\int\left(1-(1-\phi(Mt))^{1+2/\kappa}\right)dt\\
&=O\left(\frac{1}{M}e^{-M(1+2/\kappa)}\right).
\end{align*}
For the second term, notice that
\begin{align*}
    \int_0^1\int_{1/2}^{1/2 + b_\omega(x_1)}\left[\frac{1/2+ e^{-M}-x_2}{c_2}\right]^{1/\kappa}dx_2dx_1 = b_3(\omega)\left[1/2-b_\omega(x_1)\right]
\end{align*}
which yields
\begin{align*}
    b_3(\omega_{11})&=\frac{1}{1/2-b_{\omega_{11}}(x_1)}\int_0^1\int_{1/2}^{1/2 + b_{\omega_{11}}(x_1)}\left[\frac{1/2+  e^{-M}-x_2}{c_2}\right]^{1/\kappa}dx_2dx_1\\
    &\le \frac{M c_2^{-1/\kappa}}{1/2- e^{-M}}\int_0^1\int_{ e^{-M}(1-\phi(Mx_1))}^{ e^{-M}}
    u^{1/\kappa}dudx_1\\
    &=\frac{Mc_2^{-1/\kappa}}{(1/2- e^{-M})(1+1/\kappa)}
     e^{-M(1+1/\kappa)}\int(1-(1-\phi(Mt))^{1+1/\kappa})dt\\
    &\le\frac{c_2^{-1/\kappa}}{(1/2- e^{-M})(1+1/\kappa)}
    e^{-M(1+1/\kappa)}\\
    &=O\left(e^{-M(1+1/\kappa)}\right). 
\end{align*}
Hence, $|b_3(\omega_{11})-b_3(\omega_{10})| = O\left(e^{-M(1+1/\kappa)}\right) $. 
Unifying the above, we have
\begin{align*}
    H^2(P_{10}, P_{11})&=O\left(\frac{1}{M}e^{-M(1+2/\kappa)} \vee e^{-M(2+2/\kappa)}\right) \\
    &=O\left(\frac{1}{M} e^{-M(1+2/\kappa)}\right).
\end{align*}
Now choose $M$ as the smallest integer such that
$$M\ge\frac{\kappa}{\kappa+2}\log n.$$
Then we have $ H^2(P_{10}, P_{11})\le C^* n^{-1}\left(1+ o(1)\right)$ for some constant $C^*$ depending only on $\kappa, c_2, \phi$, and 
$$\int\min\{dP_{j1},dP_{j0}\}\ge 1/2\left[1-\frac{C^*}{2}n^ {-1}(1+o(1))\right]^n\ge C_1^*$$
for $n$ large enough and $C_1^*$ is another absolute constant depending only on $C^*$. Thus for $n$ large enough,
\begin{align*}
    S\ge \frac{1}{4}C_1^* e^{-M}\int \phi(t)dt\ge C_2^*n^{-\frac{\kappa}{\kappa+2}},
\end{align*}
in which the constant $C_2^*$ only depends on $\kappa, c_2$ and $ \phi$.

Combining all the results so far we get that
\begin{align*}
\liminf_{n\to\infty}\inf_{\tilde{G}_{n}}\sup_{(p,q)\in \cD^{{\tilde{\cG}_{n}}}}
n^{\frac{\kappa}{\kappa+2}}\EE[d_{\triangle}(\tilde{G}_{n},G^*)]>0,
\end{align*}
which holds when $d=2$. Using Lemma \ref{lemma:distance_n}, we have
\begin{align*}
\liminf_{n\to\infty}\inf_{\tilde{G}_{n}}\sup_{(p,q)\in \cD^{\tilde{\cG}_{n}}}
n^{\frac{\kappa+1}{\kappa+2}}\EE[d_{p,q}(\tilde{G}_{n},G^*)]>0.
\end{align*}
Using the same argument as in the proof of Theorem \ref{thm:ts}, we get $\kappa=1$, which will give us the rate $2/3$.

The proof for general $d$ can be derived similarly. We treat the last dimension $x_d$ as $x_2$ in the $d=2$ case and treat $\bx_{-d}:=(x_1,\cdots, x_{d-1})$ as $x_1$ in the $d=2$ case. Define
\begin{align*}
    q_0(\bx)= &(1-\eta_0-b_1)\II\{0\le x_d<1/2\} + \II\{1/2\le x_d<1/2+ e^{-M}\} \\
    &+(1+\eta_0+b_2)\II\{1/2+ e^{-M}\le x_d\le 1\},
\end{align*}
and 
\begin{align*}
    p_\omega(\bx) =&  1+\left[\frac{1/2+  e^{-M}-x_2}{c_2}\right]^{1/\kappa}\II\{1/2\le x_d\le1/2+\bb_\omega(\bx_{-d})\}\\
    &-b_3(\omega)\II\{1/2+\bb_\omega(\bx_{-d})< x_d\le 1\},
\end{align*}
where $\bb_\omega(\bx_{-d})$ is constructed similarly as a shallow ReLU neural network that
\begin{align*}
    \bb_\omega(\bx_{-d})&= \sum_{j_1,\cdots,j_{d-1}=1}^{M} \omega_{j_1,\cdots,j_{d-1}}\psi_{j_1,\cdots,j_{d-1}}(\bx_{-d}),
\end{align*}
where $\omega_{j_1,\cdots,j_{d-1}}$ are binary $0,1$ variables and  $$\psi_{j_1,\cdots,j_{d-1}}(\bx_{-d})=e^{-M}\boldsymbol{\phi}\left(M\left[\bx_{-d}-\left(\frac{j_1-1}{M},\cdots, \frac{j_{d-1}-1}{M}\right)\right]\right),$$
where $\boldsymbol{\phi}(\cdot)$ 
is a shallow ReLU neural network with input dimension $d-1$ satisfying the following conditions: 
\begin{itemize}
    \item 
    $\boldsymbol{\phi}=0$ outside $[-1,1]^d$ and
    $\boldsymbol{\phi}\le 1$ on $[-1,1]^d$;
    \item 
    $\max_{\bx_{-d}\in[-1,1]^d}{\boldsymbol{\phi}(\bx_{-d})}\le 1$ and $\boldsymbol{\phi}(\boldsymbol{0})=1$.
\end{itemize}
Such a construction is similar to the ``spike'' function in \cite{yarotsky2019phase} and it requires $O(d^2)$ non-zero weights. The rest of the proof follows the $d=2$ case.
\end{proof}

\section{Training with Surrogate Loss}
In this section, we consider deep classifiers trained under the hinge loss $\phi(z) = (1-z)_+ = \max\{1-z, 0\}$. This kind of surrogate loss has been widely used for ``maximum-margin'' classification, most notably for support vector machines \citep{cortes1995support}.
An desirable property of hinge loss is that its optimal classifier coincides with that under 0-1 loss \citep{lin2002support}, i.e. $f^*_\phi(\bx)= C^*(\bx)$. Hence, a lot of arguments for 0-1 loss can be easily carried over. Additionally, minimizing the sample average of an appropriately behaved loss function has a regularizing effect \citep{bartlett2006convexity}. It is thus possible to obtain uniform upper bounds on the risk of a function that minimizes the empirical average of the loss $\phi$, even for rich classes that no such upper bounds are possible for the minimizer of the empirical average of the 0–1 loss.

Under the surrogate loss, our requirement on the size of the teacher network is relaxed from (A2) as follows:  
\begin{enumerate}
    \item[(A2$_\phi$)]
     $p(\bx) - q(\bx)$ is representable by some teacher ReLU DNN $f^*_n\in\cF^*_n$ with 
    $$N_n^*=O\left(\log n\right)^{m_*},\quad L_n^*=O\left(\log n\right), \quad B_n^*, F_n^*=O(\sqrt{n})$$
    for some $m_*\ge 1$.
\end{enumerate}  

The following theorem says that the same un-improvable rate can be obtained for the empirical hinge loss minimizer $\hat{f}_{\phi,n}\in\cF_n$. 
\begin{theorem}
\label{thm:surrogate}
Suppose the underlying densities $p$ and $q$ satisfy assumptions (A1), (A2$_\phi$), (A3) and denote all such $(p,q)$ pairs as $\tilde{\cF}^*_n$. 
Let $\cF_n$ be a student ReLU DNN family with 
$L_n=O(\log n), N_n=O(\log n)^m$ and $B_n, F_n=O(\log n)$ for some $m\ge m_*$. Assume the student network is larger than the teacher network, i.e., $L_n\ge L_n^*, S_n\ge S_n^*, N_n\ge N_n^*, B_n\ge B_n^*, F_n\ge F_n^*$. Then the excess risk for $\hat{f}_{\phi,n}\in\cF_n$ satisfies 
$$\sup_{(p, q)\in\tilde{\cF}^*_n}\EE[\cE(\hat{f}_{\phi,n},C^*_n)]\asymp 
	\tilde{O}_d\left(\frac{ 1}{n}\right)^{\frac{2}{3}}  $$
\end{theorem}
Similarly, results in Corollary \ref{cor:dis} and \ref{cor:active} hold for the empirical hinge loss minimizer. Specifically, when $p,q$ are disjoint, the convergence rate of excess risk improves to $n^{-1}$, and all conclusions hold when the teacher network is larger but with bounded active pieces.

\begin{remark} [Network Depth]
\label{remark:deeper} Training with surrogate loss such as hinge loss, unlike 0-1 loss, doesn't involve any hard thresholding, i.e. $\II\{yf(\bx)<0\}$. As a result, to control the complexity of the student network, Lemma \ref{covering} is used instead of Lemma \ref{lemma:nn_piece}, which allows us to use deeper neural networks ($L_n=O(\log n)$) for both the student and teacher network. 
\end{remark}

\subsection{Proof of Theorem \ref{thm:surrogate}}
One important observation to be used in the proof is that the Bayes classifier under hinge loss is the same as that under 0-1 loss, i.e. $f^*_\phi(\bx)= C^*(\bx)$. To show the upper bound on excess risk convergence rate, we utilize the following lemma from \cite{kim2018fast}. Let $\eta(\bx)$ denote the conditional probability of label 1 that $\eta(\bx)=\PP(y=1|\bx)$. 

\begin{lemma}
\label{lemma:kim}
[Theorem 6 of \citep{kim2018fast}]
Let $\phi$ be the hinge loss. Assume \textup{(N)} with the noise exponent  $\kappa\in[0,\infty]$, and that following conditions (C1) through (C4) hold.
 \begin{itemize}
    \item[(C1)] For a positive sequence $a_n=O(n^{-a_0})$ as $n\to\infty$ for some $a_0>0$, there exists a sequence of function classes $\{\cF_n\}_{n\in\NN}$ such that $\cE_\phi(f_n,f^*_{\phi})\le a_n$
    for some  $f_n\in \cF_n$.
    \item[(C2)] 
    There exists a real valued sequence $\{F_n\}_{n\in \NN}$ with $F_n\gtrsim 1$ such that $\sup_{f\in\cF_n}\|f\|_\infty\le F_n$.
    \item[(C3)] There exists a constant $\nu\in (0,1]$ such that for any $f\in \cF_n$ and any $n\in \NN$,
  $$\EE\left[\left\{\phi(Yf(\bX))-\phi(Yf^*_{\phi}(\bX))\right\}^2\right]
        \le C_2F_n^{2-\nu} \{\cE_\phi(f,f^*_{\phi})\}^\nu$$
    for a constant $C_2>0$ depending only on $\phi$ and $\eta(\cdot)$.
    \item[(C4)] For a positive  constant $C_3>0$, there exists a sequence $\{\delta_n\}_{n\in\NN}$ such that 
    $$H_B(\delta_n, \cF_n, \|\cdot\|_2)\le C_3n\left(\frac{\delta_n}{F_n}\right)^{2-\nu},$$
    for $\{\cF_n\}_{n\in\NN}$ in (C1), $\{F_n\}_{n\in\NN}$ in (C2), and $\nu$ in (C3).
\end{itemize}
Let $\epsilon_n^2\asymp \max\{a_n, \delta_n\}$. Assume that  $n^{1-\iota}(\epsilon_n^2/F_n)^{(\kappa+2)/(\kappa+1)}\gtrsim 1$ for an arbitrarily small constant $\iota>0$. Then, the empirical $\phi$-risk minimizer $\hat{f}_{\phi,n}$ over $\cF_n$ satisfies 
\[
\EE\left[\cE(\hat{f}_{\phi,n}, C^*)\right]\lesssim \epsilon_n^2.
\]
\end{lemma}
In Lemma \ref{lemma:kim}, condition (C1) guarantees the approximation error of $f_n$ to $f_{\phi}^*$ to be sufficiently small. 
For condition (C3), we introduce the following lemma, which is reminiscent of Lemma \ref{lemma:distance_n} in the sense that it characterizes the relationship between the $\cE_\phi(f,f^*_\phi)$ and the some other distance measure between $f$ and $f_\phi^*$. 
\begin{lemma}[Lemma 6.1 of \cite{steinwart2007fast}]
\label{lemma:hinge}
Assume \textup{(N)} with the Tsybakov noise exponent $\kappa\in[0,\infty]$. Assume $\|f\|_\infty\le F$ for any $f\in \cF$. Under the hinge loss $\phi$, for any $f\in \cF$,
\begin{align*}
&\EE\left[\left(\phi(Yf(\bx))-\phi(Yf^*_\phi(\bx))\right)^2\right]\\
&\le C_{\eta, \kappa}(F+1)^{(\kappa+2)/(\kappa+1)}\left(\EE\left[\phi(Yf(\bx))-\phi(Yf^*_\phi(\bx))\right]\right)^{\kappa/\kappa+1},
\end{align*}
where $C_{\eta, \kappa}=\left(\|(2\eta-1)^{-1}\|_{\kappa,\infty}^\kappa+1\right)\II(\kappa>0)+1$ and $\|(2\eta-1)^{-1}\|_{\kappa,\infty}^\kappa$ is defined by 
$$\|(2\eta-1)^{-1}\|_{\kappa,\infty}^{\kappa}=\sup _{{t>0}}\left(t^\kappa\Pr\left(\{\bx:|(2\eta(\bx)-1)^{-1}|>t\}\right)\right).$$
\end{lemma}

For condition (C4) in Lemma \ref{lemma:kim}, we present the following lemma.
\begin{lemma}
\label{covering}
[Lemma 3 in \citet{suzuki2018adaptivity}] For any $\delta>0$, the covering number of $\cF^{DNN}(L, N, S, B)$ (in
sup-norm) satisfies
\begin{align*}
    &\log \cN(\delta, \cF^{DNN}(L, N, S, B), ||\cdot||_{\infty})\\ 
    &\le 2L(S+1)\log(\delta^{-1}(L+1)(N+1)(B\vee 1)).
\end{align*}
\end{lemma}

\begin{proof}[proof of Theorem \ref{thm:surrogate}]
The lower bound directly follows from Theorem \ref{thm:lb}, as the constructed ReLU neural network in the proof also satisfy assumption (A2$_\phi$). 

For the upper bound on the convergence rate, we utilize Lemma \ref{lemma:kim} and check the conditions (C1) through (C4). 
Since the student network is larger than the teacher, (C1) and (C2) trivially hold with arbitrarily small $a_n$ and $F_n=O(\log n)$ as assumed. To apply Lemma \ref{lemma:hinge}, notice that $C_{\eta, \kappa}= O(c_n)=O(\log n)^{m^*d^2 L_n^*}$ by assumption (A3) and $F=O(\log n)$, we have (C3) holds for $\nu=\kappa/(\kappa+1)+\epsilon_n$, where $\epsilon_n=(2+m^*d^2L_n^*)\log\log n/\log n$. The term $\epsilon_n$ is to deal with the fact that $C_{\eta, \kappa}$ can also diverge at an $O(\log n)^{m^*d^2L_n^*}$ rate.

For (C4), by Lemma \ref{covering}, 
    \begin{eqnarray*}
    &&\log \cN(\delta_n, \cF^{\textup{DNN}}(L_n, N_n, S_n, B_n, F_n),\|\cdot\|_\infty) \\
    &&\le 2L_n(S_n+1)\log\left(\delta_n^{-1}(L_n+1)(N_n+1)(B_n\vee1)\right) \\
    &&\lesssim (\log n)^{2m+2} \log\left(\delta_n^{-1}\vee \log^m(n)\right).
    \end{eqnarray*}
Therefore, (\ref{lemma:kim}) implies that
  (C3) is satisfied if we choose $\delta_n$ with
$$\delta_n^{\frac{\kappa+2}{\kappa+1}}\gtrsim \frac{(\log n)^{2m+2+(\kappa+2)/(\kappa+1)+2+m^*d^2L_n^*+1}}{n},$$
which can be satisfied by choosing
$$\delta_n = 
\left(\frac{(\log n)^{2m+m^*d^2L_n^* + 7} }{n}\right)^{\frac{ \kappa+1}{\kappa+2}}.$$
Similar to the proof of Theorem \ref{thm:ts}, the Tsybakov exponent $\kappa=1$. Thus, by Lemma \ref{lemma:kim} with $\epsilon_n^2=\delta_n$, the proof of Theorem \ref{thm:surrogate} is completed.
\end{proof}


\section{Discussion and Future Directions}
In this work, we obtain a sharp rate of convergence for the excess risk under both empirical 0-1 loss and hinge loss minimizer in the teacher-student setting. Our current results for   training under 0-1 loss only hold for student networks with $O(1)$ layers and the assumption that $f^*_n\in\cF_n$, i.e. zero approximation, is required. In the future, we aim to relax these two constraints and provide more comprehensive analysis of the teacher-student network. Additionally, we would like to 
\begin{itemize}
    \item
    explore other type of neural networks such as convolutional neural network and residual neural network, which are both very successful at image classification;
    \item 
    consider the implicit bias of training algorithms, e.g. stochastic gradient descent, to regularize the complexity of larger and deeper neural networks in the teacher-student setting; 

    \item
    consider the more general improper learning scenario where the Bayes classifier is not necessarily in the student neural network;
    
    \item
    consider other popular surrogate losses such as exponential loss or cross entropy loss.
    
\end{itemize}
Further investigation under the teacher-student network setting may facilitate a better understanding of how deep neural network works and
shed light on its empirical success especially in high-dimensional image classification.

\bibliography{references}

\newpage
\section{Appendix}
\subsection{Smooth Boundary Condition}
\label{app:boundary}
In this section we review some existing work under smooth boundary condition in details and point out its connection to our proposed teacher-student neural network.

\paragraph{Smooth Functions}
A function has H\"older smoothness index $\beta$ if all partial derivatives up to order $\lfloor \beta \rfloor$ exist and are bounded, and the partial derivatives of order $\lfloor \beta \rfloor$ are $\beta - \lfloor \beta \rfloor$ Lipschitz. The ball of $\beta$-H\"older functions with radius $R$ is then defined as
\begin{align*}
	\cH_r^\beta(R) = \Big\{ 
	&f:\mathbb{R}^r \rightarrow \mathbb{R} : \\
	&\sum_{\balpha : |\balpha| < \beta}\|\partial^{\balpha} f\|_\infty + \sum_{\balpha : |\balpha |= \lfloor \beta \rfloor } \, \sup_{\stackrel{\bx, \by \in D}{\bx \neq \by}}
	\frac{|\partial^{\balpha} f(\bx) - \partial^{\balpha} f(\by)|}{|\bx-\by|_\infty^{\beta-\lfloor \beta \rfloor}} \leq R
	\Big\},
\end{align*}
where $\partial^{\balpha}= \partial^{\alpha_1}\ldots \partial^{\alpha_r}$ with $\balpha =(\alpha_1, \ldots, \alpha_r)\in \mathbb{N}^r$ and $|\balpha| :=|\balpha|_1.$

\paragraph{Boundary Assumption} 
It is known that estimating the classifier directly instead of the conditional class probability helps achieve fast convergence rates 
\citep{mammen1999smooth, tsybakov2004optimal, tsybakov2005square, audibert2007fast}. Classification in this case can be thought of as nonparametric estimation of sets where we directly estimate the decision regions for different labels, e.g., $G$ for label 1. Then, the classifier is determined by attributing $\bx$ to label 1 if $\bx\in G$ and to label $-1$ otherwise, i.e.,
$$C(\bx) = 2\cdot \II_{G}(\bx)-1.$$
In this case, the Bayes risk can be written as
$$R(G) = 1/2\left(\int_{G^c}p(\bx)\QQ(d\bx)+\int_G q(\bx)\QQ(d\bx)\right).$$

Denote $G^* = \{\bx:p(\bx)\ge q(\bx)\}$ to be the Bayes risk minimizer, 
and the classification problem is equivalent to estimation of the optimal set $G^*$. 
Given $p,q$, let 
$$\cX_+ = \{\bx\in\cX: p(\bx)\ge q(\bx)\},\quad \cX_- = \{\bx\in\cX: p(\bx)<q(\bx)\}$$
The optimal decision rule is to assign label 1 to $\bx\in\cX_+$ and $-1$ to $\bx\in\cX_-$. The decision boundary in this case is $\{\bx\in\cX: p(\bx)=q(\bx)\}$. To characterize the smoothness of the boundary, it is usually assumed that $\cX_+$ consists of union and intersection of smooth hyper-surfaces \citep{kim2018fast, tsybakov2004optimal}. 
Specifically, the following assumption is widely adopted \citep{mammen1999smooth, tsybakov2004optimal, kim2018fast,imaizumi2018deep} and referred to as "boundary fragment".
Let $\cH$ be some smooth function space from $\RR^{d-1}\to \RR$. Define sets $\cG^{\cH}$ as 
\begin{equation}
\label{eqn:gh}
    \cG^{\cH} = \{\bx\in \cX: x_j>h(\bx_{-j}), h\in\cH, j \in\{1, 2,\cdots, d\}\}
\end{equation}
where $\bx_{-j} = (x_1, \cdots, x_{j-1}, x_{j+1}, \cdots, x_d)$. It is assumed that $\cX_+$ is composed of finite union and intersection of sets in $\cG^{\cH}$. The seemingly odd form (\ref{eqn:gh}) enforces special structures on the indicator function and reduces the complexity of the corresponding sets.

A more general assumption on the decision boundary is that the set, denoted as $\cG$, containing all possible $\cX_+$ cannot be too large. This is measured by the bracketing entropy $H_B$ of the metric space $(\cG, d_\triangle)$.
The more general assumption of the decision boundary is stated as
\begin{itemize}
\item[(B)] There exists $A>0$ and $\rho\in[0,\infty]$ such that 
$$H_B(\delta, \cG, d_\triangle)\le A \delta^{-\rho}.$$
\end{itemize}
Our proposed teacher-student network setting follows this more general assumption and makes more sense in high-dimensional classification.
\begin{remark}
\label{remark:beta}
If $\cH$ are $\beta$-smooth functions in (\ref{eqn:gh}), then (B) holds with $\rho = (d-1)/\beta$. On the other hand, Lemma \ref{lemma:nn_piece} gives $H_B\lesssim A_n\vee (B_n\log(1/\delta))$ where in terms of $\delta$, the order is $\log(1/\delta)$, which corresponds to $\rho\to 0$ and $\beta\to\infty$. However, this doesn't necessarily imply the boundary fragment set (\ref{eqn:gh}) is larger, since $A_n,B_n$ depend on $n$ and are not absolute constants as in (B).
\end{remark}

\subsection{Comments on Lemma \ref{lemma:nn_piece}}
\label{app:entropy}
Lemma \ref{lemma:nn_piece} is the main result for controlling the bracketing entropy of the estimation sets. Below we point out some key property of this result and compare to other entropy bounds of neural networks. 
\paragraph{Exponential Dependence on Depth}
The bracketing entropy of $\cG^{\cF}$ developed in Lemma \ref{lemma:nn_piece} is much larger than that of $\cF$ itself with respect to $\|\cdot\|_\infty$, as described in Lemma \ref{covering}. 
The main difference is the dependence on the number of layers $L$: the dependence is linear in Lemma \ref{covering} while exponential in Lemma \ref{lemma:nn_piece}. Thus, even though $\cG^{\cF}$ is a slice of the subgraph of $\cF$, $\cG^{\cF}$ is much more complicated than $\cF$ in term of entropy. We argue that this gap cannot be closed even in the special case $d=1$. 

\begin{figure}
    \centering
    \includegraphics[scale=0.5]{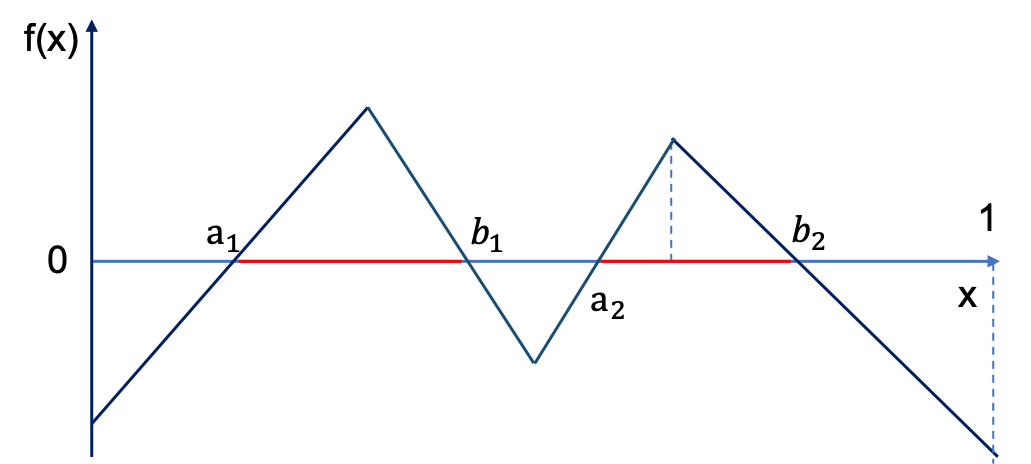}
    \caption{Example of a ReLU function in 1D. The induced set where $f>0$ is colored red and it's a union of two intervals $(a_1, b_1), (a_2, b_2)$. All pieces cross 0 so there are all active.}
    \label{fig:1d}
\end{figure}

\cite{montufar2014number} establish a lower bound on the maximum number of linear pieces for a ReLU neural network (Lemma \ref{lemma:lb}). Consider a 1-dimensional ReLU DNN function with $L$ layers and 2 nodes on each layer. Corollary 5 of  \cite{montufar2014number} show that there exists some $f$ with $s=\Omega(2^{L-1})$ pieces on $[0,1]$. With scaling and shifting, assume that on each piece the linear function crosses 0. Then, $G^f$ will be at least $\lfloor s/2\rfloor=\Omega(2^{L-2})$ intervals. Denote these disjoint intervals to be $\{(a_i, b_i)\}_{i=1}^{\lfloor s/2\rfloor}$.  
Since they are disjoint, to construct a $\delta$-bracket of all the intervals, we need to $\delta$-cover all the $a_i$'s and $b_i$'s. Similar to the grid argument from the proof of Lemma \ref{lemma:bracketing}, we need at least 
$$\binom{1/\delta}{s}=\Omega\left((1/\delta-s)^{s}\right)$$ different combinations of the $s$ grid points. Hence the bracketing entropy must be in the order of $$\log ((1/\delta-s)^{s})=2^{L-2}\log(1/\delta).$$
The exponential dependence of depth $L$ in the entropy stems from the fact that the number of linear regions of ReLU DNNs scales exponentially with $L$. 

\paragraph{Independent of Weights Magnitude}
We also want to point out that the entropy of $\cG^{\cF}$ is not concerned with the magnitude of the neural network weights, in contrast to the bound in Lemma \ref{covering}. This is because any scaling of the function doesn't change how it intercepts with zero. Hence, unlike $\cF$, the entropy of $\cG^{\cF}$ doesn't depend on the weight maximum $B$. 

\paragraph{The Use of ReLU Activation}
The reason why we can even bound the entropy of $\cG^{\cF}$ critically relies on the fact that we are considering the ReLU activation function. If we consider smooth nonlinear activation functions, e.g. hyperbolic tangent, sigmoid, instead of the order $\log(1/\delta)$, we can only get the entropy of a much larger order $$H_B(\delta, \cG^{\cF}, d_\triangle)\le A\delta^{-\alpha}$$ 
for some constant $A>0$ and $\alpha>0$. To see this, consider the case $d=2$. Instead of polygons, which can be controlled by the vertices, the regions have smooth boundary and will require $O(1/\delta)$ many grid points to cover. Thus the covering number is of order $$\binom{1/\delta^2}{1/\delta} = O\left(\left(\frac{1}{\delta}\right)^{2/\delta}\right).$$
Thus, the entropy is in a polynomial order of $1/\delta$.

\subsection{Illustration and Proof of Lemma \ref{lemma:a3}}
\label{app:a3}
In this section, Assumption (A3) will be examined in the setting that the teacher network $f_n^*$ has random weights. 
We will argue that with probability at least $1-\delta$, $f_n^*$ will satisfy assumption (A3) with $T_n=A(\delta)/(\log n)^{m^*d^2L_n^*}$ and $c_n=B(\delta)(\log n)^{m^*d^2L_n^*}$, where $A(\delta), B(\delta)$ are constants depending only on $\delta$ and the distribution of the random weights, e.g. normal, truncated normal, etc. Hence, the results which assume Assumption (A3)  will hold with high probability. 

\paragraph{A Toy Case}
To illustrate the intuition, consider the case where $d=1$ and $f_n^*$ is the following one hidden-layer ReLU neural network
\begin{equation}\label{toy:teacher}
f_n^*(x) = \sum_{j=1}^{N_n^*}w_{2j}\sigma(w_{1j}x+b_j)+b,\quad x\in[0,1],
\end{equation}
with $L_n^*=1$, $N_n^*=O(\log n)$ and $w_{1j},w_{2j},b_j,b$ are i.i.d. standard Gaussian. 
Since all the weights are almost surely nonzero, we omit the zero weight cases for the analysis. 
Let $p_i = (u_i,v_i)$, $i =1,2,\ldots, s$, denote the active pieces of (\ref{toy:teacher}).
By Lemma \ref{lemma:upper_bound_improved}, we know that $s=O(\log n)$. For each $p_i$, define the following quantities:
\begin{enumerate}
    \item 
    $k_i$ = the slope of $f_n^*(x)$ on $x\in p_i$;
    \item
    $t_i$ = $\max_{x\in p_i}f_n^*(x) \wedge \max_{\bx\in p_i}-f_n^*(x) $.
\end{enumerate}
See Figure \ref{fig:random} for an illustration.
Then, assumption (A3) is satisfied if  
\begin{equation}\label{toy:examinations}
\textrm{$\min_i\{|k_i|\}=\Omega(1/\log n)$ \ \ and \ \  $\min_i\{t_i\}=\Omega(1/\log n)$.} 
\end{equation}
Next we will rigorously examine (\ref{toy:examinations}). 

From (\ref{toy:teacher}), each $k_i$ can be expressed as $w_{1j}w_{2j}$ for some $j\in\{1,2,\cdots,N_n^*\}$. Therefore, 
$\min_{1\le i\le N_n^*}\{|k_i|\} = \min_{1\le j\le N_n^*}\{|w_{1j}w_{2j}|\}$.
Since $w_{1j},w_{2j}$ are i.i.d. standard Gaussian, we have
\begin{align*}
    &\PP(\min_{1\le i\le N_n^*}\{|k_i|\}< k) =  \PP(\min_{1\le j\le N_n^*}\{|w_{1j}w_{2j}|\} < {k})\\
   &\le \sum_{j=1}^{N_n^*}\PP(|w_{1j}w_{2j}|< {k})
    \le 2N_n^* \PP(|w_{11}| < \sqrt{k})\le 2N_n^*\sqrt{k}.
\end{align*}
By choosing 
$k=\left(\frac{\delta}{2N_n^*}\right)^2,$
we have $\min_{1\le i\le N_n^*}\{|k_i|\}=\Omega(1/\log n)$ with probability at least $1-\delta$.

\begin{figure}
    \centering
    \includegraphics[scale=0.5]{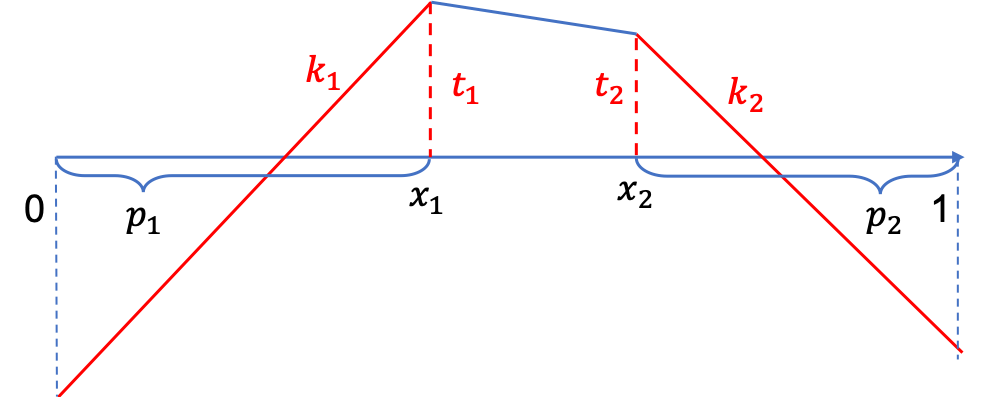}
    \caption{Example of a ReLU function in $[0,1]$. There are two active pieces $p_1, p_2$. On each active piece, $t_i. k_i$ are illustrated in color red. }
    \label{fig:random}
\end{figure}

On the other hand, 
for any $i=1,\ldots,s$, $t_i=|f_n^*(x_{h_i})|$ for some $h_i\in\{1,\cdots, N_n^*\}$,
where $x_{h_i}=-b_{h_i}/w_{1h_i}$. 
Hence
\[
\min_{1\le i\le s}\{t_i\}\ge\min_{1\le j\le N_n^*}\{|f_n^*(x_j)|\}.
\]
Let $W_1=\{w_{1j},b_j\}_{j=1}^{N_n^*}$. Then,
$
    f_n^*(x_i)\given W_1\sim N(0,\sigma_{x_i}^2),
$
where $\sigma_{x_i}^2$ has an expression of $\sum_{j=1}^{N_n^*} \sigma(w_{1j}x_i+b_j)^2 + 1$.
Hence, for any $t>0$, 
\begin{align*}
    \PP(\min_{i\le N_n^*}\{|f_n^*(x_i)|\}< t \given W_1) &\le \sum_{i=1}^{N_n^*}\PP(|f_n^*(x_i)|< t\given W_1)\\
    &= N_n^* \PP(|f_n^*(x_i)| < t \given W_1) \le N_n^*\left(\frac{t}{\sigma_{x_i}}\right).
\end{align*}
Since $\sigma_{x_i}\ge 1$, by taking $t = \delta/N_n^*$, we have that with probability at least $1-\delta$, $\min_{i}\{t_i\}\ge t$ and $t=\Omega(1/\log n).$
Therefore, (\ref{toy:examinations}) holds with high probability, so that assumption (A3) holds by setting $1/c_n = \min_i\{|k_i|\}$ and $T_n = \min_i\{t_i\}$, which are both in the order of $\Omega(1/\log n)$.

\paragraph{General Case}
Now we consider the general case $d>1$ and $L^*_n>1$. The teacher network has an expression
\[
f_n^*(\bx) = \bW^{(L_n^*+1)}\sigma_{(\bW^{(L_n^*)},\bb^{(L_n^*)})}\circ\cdots\circ \sigma_{(\bW^{(1)},\bb^{(1)})}(\bx) + \bb^{(L_n^*+1)}, \bx\in [0,1]^d.
\]
Let $N_n^* = O(\log n)^{m^*}$. By Lemma \ref{lemma:upper_bound_improved}, $f_n^*$ has 
linear pieces $p_1,\ldots,p_s$ for $s=O(\log n)^{m^*L_n^*d}$. 
Let $\{\bx_i,\bx_2, \ldots, \bx_{v_s}\}$ be the collection of vertices of  $\{p_1,\ldots,p_s\}$. We call such $\bx_i\in\RR^d$ a \textit{piece vertex} and it's not the same as the vertex of $\{\bx\in\cX:f_n(x)\ge 0\}$, which is closely examined in the proof of Lemma \ref{lemma:nn_piece}.
The following lemma states that $v_s=O(\log n)^{m^* L_n^* d^2}$ in our setting. 
\begin{lemma}
\label{lemma:vertex}
Let $f$ be a ReLU neural network with $d$-dimensional input, $L$ hidden layers and width $N$ for every layer. Then, $v_s=O(N)^{Ld^2}$.
\end{lemma} 
\begin{proof}
Recall that $\bw^{(l)}_i$ and $b^{(l)}_i$ for $i=1,\ldots,N$, $1\le l\le L$ are the weight vectors and biases on the $l$-th hidden layer. 
For $i=1,\ldots,N$, define
\[
f^{(l-1)}_{ i}(\bx) = \bw^{(l)}_i\sigma_{(\bW^{(l-1)},\bb^{(l-1)})}\circ\cdots\circ \sigma_{(\bW^{(1)},\bb^{(1)})}(\bx)+b^{(l)}_i,
\]
which maps $\RR^d\to \RR$. We can rewrite $f$ as 
\begin{align}
\label{eqn:last}
    f(\bx) = \sum_{i=1}^{N} w^{(L+1)}_{i}\sigma(f^{(L-1)}_{i}(\bx)) + b^{(L+1)},
\end{align}
In other words, $f^{(L-1)}_{ i}(\bx)$ represents the inputs to the $i$-th ReLU unit in the last hidden layer of $f$ and itself is an $(L-1)$-hidden-layer ReLU neural network. 

The key idea of the proof is by induction. Notice that the piece vertices of $f$ can only come from the following two ways:
Type I: The piece vertices of $f^{(L-1)}_{1}, f^{(L-1)}_{ 2},\ldots, f^{(L-1)}_{ N}$, in whose local neighbourhoods, the ReLU units in the last layer doesn't change sign;
Type II: By activations of the ReLU unit in the last layer. i.e. $f^{(L-1)}_{ i}(\bx)=0$ for some $i=1,\ldots,N$.
Let $V(l)$ be the maximum number of piece vertices of an $l$-hidden-layer ReLU neural network with width $N$ and let $U(l)$ be the maximum number of Type II piece vertices created at layer $l$. Then for $1<l\le L$ we have 
\begin{equation}
\label{eqn:induction}
    V(l) \le N V(l-1) + U(l).
\end{equation}

For $U(l)$, the key is to connect the Type II piece vertices of $f$ to the vertices of $\{\bx\in\cX: f^{(L-1)}_{ i}(\bx)\ge 0\}$, which has been extensively studied in Lemma \ref{lemma:nn_piece}.
To this end, we define another quantity. On the $i$-th ReLU unit in the $l$-th hidden layer, let $R^{(l)}_{i}:=\{\bx\in\cX: f^{(l)}_{ i}(\bx)= 0\}$, which consists of $(d-1)$-dimensional hyperplane segments. 
To be specific, denote all the active pieces of $f^{(l)}_{ i}(\bx)$ to be $\{p^{(l)}_{ij}: j=1,\dots,s^{(l)}_i\}$, where $s^{(l)}_i = O(N)^{(l-1)d}$ according to Lemma \ref{lemma:upper_bound_improved} for any $1\le i\le N$. On each active piece $p^{(l)}_{ij}$, denote 
\[
h^{(l)}_{ij} = \{(\bx, f^{(l)}_{ i}(\bx)): \bx\in p^{(l)}_{ij}\}\cap \{(\bx,0):\bx\in p^{(l)}_{ij}\},
\]
which is part of a $(d-1)$-dimensional hyperplane. 
Then we have $R^{(l)}_{i}=\{h^{(l)}_{ij}:j=1,\dots, s^{(l)}_i\}$, a collection of $(d-1)$-dimensional hyperplane segments. Let $R^{(l)} = \cup_{i=1}^{N}R^{(l)}_{i}$, which corresponds to the piece boundaries of $f^{l+1}$. 

By definition, all Type II pieces vertices must reside in at least one of the the activation sets ($z=0$ in $\sigma(z)$) of the ReLU units in the last layer. 
$R^{(L)}$ contains all such activation sets for the last hidden layer, i.e. for any $h\in R$, there exists $1\le i\le N$ such that $f_i(\bx)=0,\  \forall \bx\in h$. 
The Type II pieces vertices are jointly determined by such activation sets and the piece boundary of $f_i$'s (dimension $d-1$), i.e. $R^{(L-2)}_{i}$. 
Therefore, the total number of such piece vertices can be bounded by 
$$U(l)\le \binom{\abr{R^{(l-1)}}+\abr{R^{(l-2)}}}{d}=O(N)^{(l-1)d^2+d},$$
where $\abr{R^{(l)}}$ denotes the number of elements in ${R^{(l)}}$, which is bounded by $O(N)^{(l-1)d+1}$.

For $V(L)$, we first conclude that $V(1)=O(N^d)$. For a 1-hidden layer ReLU network, the decision boundary of every ReLU unit is a ($d-1$)-dimension hyperplane ($\bw_1\bx+b_1=0$). The maximum number of piece vertices is bounded by $\binom{N}{d}=O(N^d)$.  
Then, (\ref{eqn:induction}) can be repeatedly broken down to
\begin{align*}
    V(L) &\le N V(L-1) + U(L)\\
             &\le N^2 V(L-2) + NU(L-1)+ U(L)\le \cdots \\
             &\le N^{L-1}V(1) + \sum_{l=0}^{L-1} N^lU(L-l)\\
             &= O\rbr{N^{L-1+d}} + O\rbr{\sum_{l=0}^{L-1} N^{(L-l-1)d^2+d+l}}\\
             &= O\rbr{N^{(L-1)d^2+d}}=O\rbr{N^{Ld^2}}\\
\end{align*}
\end{proof}

As an extension to the toy case, for any $1\le i\le N_n^*$, define
\begin{enumerate}
    \item 
    $k_i = \min_{j=1,\ldots,d}\cbr{\abr{\frac{\partial f_n^*(\bx)}{\partial x_j}}: \bx\in p_i}$;
    \item
    $t_0$ = $\min_{1\le i\le v_s}\cbr{\abr{f_n^*(\bx_i)}}.$
\end{enumerate}
That is, $k_i$ is the minimal absolute values of the directional derivatives of $f_n^*$ on piece $p_i$.
Assumption (A3) is satisfied if the following holds:
\begin{equation}\label{nontrivial:examinations}
\textrm{$\min_{1\le i\le s}\{k_i\}, t_0=\Omega(\log n)^{m^*d^2L_n^*}$.}
\end{equation}

We will check (\ref{nontrivial:examinations}).
Since the partial derivative of $f_n^*(\bx)$ for $\bx\in p_i$ can be expressed as the product of the random weights, we have
\begin{align*}
    \min_{1\le i\le s}\{k_i\}\ge \min_{1\le j_l\le N_n^*}\abr{\prod_{l=1}^{L_n^*+1}w^{(l)}_{i_l j_{l}}},
\end{align*}
where $w^{(l)}_{i_l j_{l}}$ is an element from $\bW^{(l)}$. 
Since 
\begin{align*}
    \min_{1\le j_l\le N_n^*}\prod_{l=1}^{L_n^*+1}|w^{(l)}_{i_l j_{l}}|&\ge \min_{l\le L_n^*+1; i_l,j_l\le N_n^*}|w^{(l)}_{i_l j_{l}}|^{L_n^*+1},
\end{align*}
 we get that
\begin{align*}
   \PP( \min_{1\le i\le s}\{k_i\}<k)&\le \PP\left(\min_{l\le L_n^*+1; j_l\le N_n^*}|w_{l j_{l}}|\le k^{1/(L_n^*+1)}\right)\\
   &\le (L_n^*+1)(N_n^*)^2\PP(|w_{l j_{l}}|<k^{1/(L_n^*+1)})\\
   &\le (L_n^*+1)(N_n^*)^2k^{1/(L_n^*+1)}.
\end{align*}
By taking 
$$k_0 = \rbr{\frac{\delta}{(N_n^*)^2(L_n^*+1)}}^{L_n^*+1},$$ 
we have that with probability at least $1-\delta$, $\min_{1\le i\le s}\{k_i\}\ge k_0$ and $k_0=\Omega(1/\log n)^{2m^*(L_n^*+1)}$. 

On the other hand, for any $t_i$, there exist $j=1,\ldots,v_s$ such that $t_i=f_n^*(\bx_{j})$. 
Hence $$
\min_{i=1,\ldots,v_s}\{t_i\}\ge\min_{j=1,\ldots,v_s}\{|f_n^*(\bx_j)|\}.$$
Let $\bW_{-L_n^*}:=\{\bW^{(l)},\bb^{(l)}\}_{l=1}^{L_n^*}$. Then we have
$
    f_n^*(\bx_j)\given \bW_{-L_n^*} \sim N(0,\sigma_{\bx_j}^2),
$
where $\sigma_{\bx_j}^2$ depends on $\bW_{-L_n^*}$ and $\sigma_{\bx_j}^2\ge 1$ that 
\begin{align*}
 \sigma_{\bx_j}^2\given \bW_{-L_n^*} :&=\sum_{i=1}^{N_{L_n^*}} \sigma_i^2(\bx_j) + 1,
 \end{align*}
which is reminiscent of (\ref{eqn:last}) and $N_{L_n^*}$ is the width of the last layer and $\sigma_j(\cdot)$'s are outputs (post-activations) from the last layer given $\bW_{-L_n^*}$.
Therefore, for any $t>0$, we have
\begin{align*}
    \PP(\min_{1\le j\le v_s}\{|f_n^*(\bx_j)|\}< t\given \bW_{-L_n^*}) &\le \sum_{j=1}^{v_s}\PP(|f_n^*(\bx_j)|< t\given \bW_{-L_n^*})\\
    &= v_s \PP(|f_n^*(\bx_1)| < t\given \bW_{-L_n^*})\\
    & \le v_s\left(\frac{t}{\sigma_{x_i}}\right)\le t(N_n^*)^{d^2 L_n^*}.
\end{align*}
Thus by taking $t = \delta/(N_n^*)^{d^2 L_n^*}$, we have that with probability at least $1-\delta$, $\min_{i}\{t_i\}\ge t$ and $t=\Omega(1/\log n)^{m^*d^2 L_n^*}$. Therefore, (\ref{nontrivial:examinations}) holds. That is to say,
when $d\ge 2$, with high probability, Assumption (A3) holds in which $c_n,1/T_n = O(\log n)^{m^*d^2L_n^*}$.

Notice that the probability arguments used in this section don't rely on Gaussian distribution. As long as all weights are i.i.d. with distribution that doesn't have a point mass at 0, our claim holds. 

\end{document}